\documentclass{article}

\usepackage{microtype}
\usepackage{graphicx}
\usepackage{subcaption}
\usepackage{booktabs} % for professional tables

\usepackage{hyperref}

\usepackage[preprint]{icml2026}

\usepackage{amsmath}
\usepackage{amssymb}
\usepackage{mathtools}
\usepackage{amsthm}
\usepackage{booktabs}      % 用于 \toprule, \midrule, \bottomrule
\usepackage{multirow}      % 用于跨行
\usepackage[table]{xcolor} % 用于单元格着色
\usepackage{amsmath}       % 用于数学符号
\usepackage{pifont}
\usepackage{graphicx}
\usepackage{float}
\usepackage{subcaption}

\usepackage[capitalize,noabbrev]{cleveref}

\theoremstyle{plain}
\newtheorem{theorem}{Theorem}[section]

\theoremstyle{definition}

\theoremstyle{remark}

\usepackage[textsize=tiny]{todonotes}
\usepackage{wrapfig}
\usepackage{booktabs}
\begin{document}
\twocolumn[
  \icmltitle{Rethinking Representativeness and Diversity in Dynamic Data Selection}




  \begin{icmlauthorlist}
    \icmlauthor{Yuzhe Zhou}{yyy}
    \icmlauthor{Zhenglin Hua}{yyy}
    \icmlauthor{Haiyun Guo}{comp}
    \icmlauthor{Yuheng Jia}{yyy}
  \end{icmlauthorlist}

  \icmlaffiliation{yyy}{Southeast University, Nanjing, China}
  \icmlaffiliation{comp}{Chinese Academy of Sciences, Beijing, China}
  \icmlcorrespondingauthor{Haiyun Guo}{haiyun.guo@nlpr.ia.ac.cn}
  \icmlcorrespondingauthor{Yuheng Jia}{yhjia@seu.edu.cn}

  \icmlkeywords{Machine Learning, ICML}

  \vskip 0.3in
]
\printAffiliationsAndNotice{}
\begin{abstract}
Dynamic data selection accelerates training by sampling a changing subset of the dataset while preserving accuracy. We rethink two core notions underlying sample evaluation: representativeness and diversity. Instead of local geometric centrality, we define representativeness as coverage of dataset-level common or high-frequency feature factors. Instead of within-subset dispersion, we define diversity at the process level, requiring the selection trajectory to gradually include complementary rare factors over training. Based on this view, we propose a dynamic selection framework with three components. First, we score representativeness in a plug-in feature space to prioritize samples covering frequent factors. We instantiate this with a sparse autoencoder trained on the target dataset, using sparse unit activations to summarize both individual samples and dataset-wide factor statistics. Second, we realize process-level diversity by combining rare-factor sampling with a Usage-Frequency Penalty that promotes sample rotation, provably discourages monopoly, and reduces gradient bias. Third, we couple the two-dimensional scoring with a smooth scheduler that transitions selection from core-pattern consolidation to rare-factor exploration, without extra gradients, influence estimates, or second-order computations on the training model. Extensive experiments on five benchmarks across vision and text tasks demonstrate improved accuracy–efficiency trade-offs across models. Our method matches or exceeds full-data accuracy with over 2× training acceleration. \emph{Code will be released.}

\end{abstract}    
\section{Introduction}
\label{sec:intro}

\begin{figure}[t]
    \centering
    \includegraphics[width=1\linewidth]{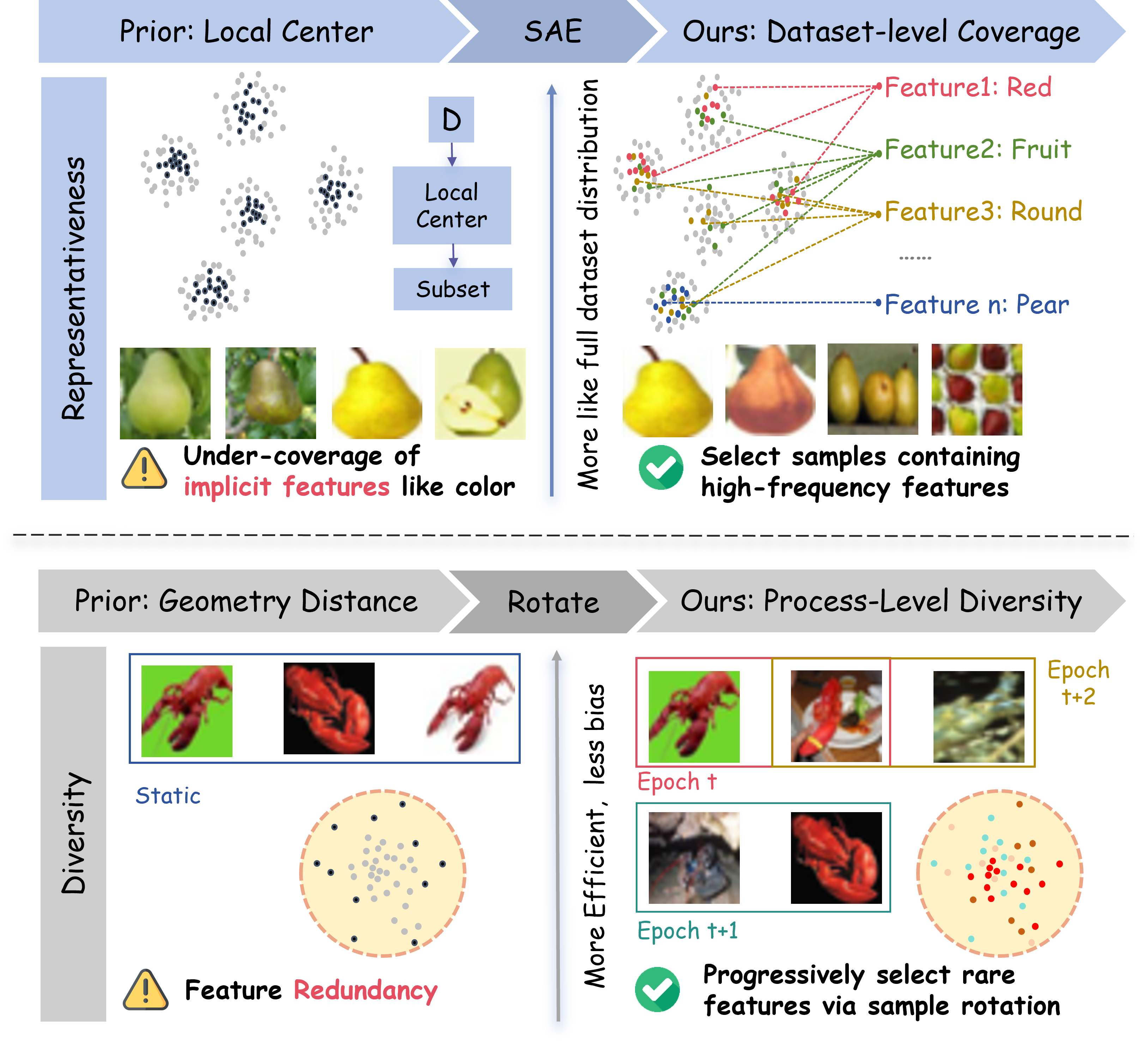}
    \caption{\textbf{Conceptual comparison with prior data selection methods.}
Previous methods relying on 
geometry-based metrics may overemphasize local centrality and miss implicit feature factors. Our framework measures representativeness as dataset-level coverage of common or high-frequency factors, and enforces process-level diversity by promoting sample rotation across epochs rather than optimizing a single static subset.
}
%     \caption{\textbf{Conceptual comparison with prior data selection.}
% Geometric heuristics may miss informative distributional variations when estimating representativeness and can introduce redundancy when enforcing diversity.
% Our framework instead measures representativeness via (weighted) coverage of common/high-frequency feature factors, and enforces process-level diversity with sample rotation to cover complementary rare patterns over time and reduce selection bias.}
    \label{fig:111}
\end{figure}

The remarkable performance of deep learning models often comes at the cost of prohibitive computation overhead from training on massive datasets \cite{10655294,touvron2023llamaopenefficientfoundation,touvron2021trainingdataefficientimagetransformers}. Data selection aims to extract a small yet high-value subset from a large corpus, thereby reducing cost and wall-clock time while maintaining or even improving accuracy  \cite{albalak2024surveydataselectionlanguage}.

 Early work predominantly adopted static data selection strategies \cite{xie2023dataselectionlanguagemodels,10.1145/1007352.1007400,sener2018activelearningconvolutionalneural,mirzasoleiman2020coresetsdataefficienttrainingmachine}, i.e., choosing a fixed subset before training that balances representativeness and diversity \cite{yang2024clip,agarwal2020contextualdiversityactivelearning,ge2024clusteringrankingdiversitypreservedinstruction,yu2025masteringcollaborativemultimodaldata,du2025disentanglingrolesrepresentationselection}. Its fundamental limitation is a mismatch between a single, static view of the data distribution and the model’s evolving capacity and learning needs across training. In practice, to preserve final accuracy, the subset size usually cannot be reduced aggressively (e.g., performance may degrade significantly unless about 70\% of the data is retained \cite{Tan2023DataPV,xia2023moderate,yang2024clip,mirzasoleiman2020coresetsdataefficienttrainingmachine}), which sharply limits potential speedups. 
This motivated dynamic data selection methods \cite{raju2021acceleratingdeeplearningdynamic,qin2024infobatch,hassan2025rcap,Zheng2025RSD15KAL} that adaptively select different samples according to the model’s current state (e.g., loss, gradients).

However, existing methods face two fundamental challenges in estimating representativeness and diversity. First, many approaches rely on local geometric proxies in a feature space. The representativeness is approximated by distance- or centroid-coverage criteria, while diversity is encouraged by enforcing dispersion or low similarity within a selected subset. Although convenient, these objectives primarily preserve local neighborhood structure and do not explicitly control dataset-level coverage of informative factors. As a result, a subset can cover cluster centers yet still under-cover common but important attribute factors (e.g., global color distribution or material-related latent attributes) that are not aligned with geometric centrality. Second, dynamic selection often re-scores examples using instantaneous model signals and samples greedily without process-level constraints. Under noisy scores, a small set of high-scoring instances can be repeatedly chosen, creating a sample-monopoly effect that shifts the effective sampling distribution over time and yields biased gradient estimates relative to full-data risk minimization. 
% However, current methods face two fundamental problems when estimating sample representativeness and diversity. First, many approaches rely on local geometric proxies to quantify sample value. Representativeness is often approximated by distance- or coverage-based criteria, and diversity is enforced by dispersion or similarity constraints within a subset. While convenient, these proxies mainly preserve local neighborhood structure in the feature space. Consequently, the selected subset may not match dataset-level coverage of informative factors. In particular, clustering provides no explicit control over global attribute coverage (e.g., the overall distribution of colors or material-related latent attributes),which can cause common yet important patterns to be missed. Second, dynamic selection typically re-scores examples using instantaneous model signals and samples greedily without process-level constraints. When such signals are noisy, a small set of high-scoring instances can be repeatedly selected, creating a sample-monopoly effect that shifts the effective sampling distribution and induces biased gradient estimates away from full-data risk minimization. 

To address the above issues, we propose a dynamic selection framework that scores samples in a plug-in feature space and enforces process-level diversity control throughout training. The key idea is to estimate representativeness as dataset-level coverage of high-frequency feature factors, while promoting diversity through rotation over time so that the cumulatively selected data progressively includes complementary rare factors during training. Specifically, our framework consists of three tightly coupled modules. First, we build a representativeness estimator from sparse unit activations in the chosen feature space: we train a sparse autoencoder to obtain sparse activation patterns, and score a sample by how well it covers frequently activated factors, rather than by local geometric centrality. Second, we operationalize process-level diversity by quantifying factor rarity and introducing a usage-frequency penalty that discourages repeated selection, prevents sample monopoly, and mitigates the resulting gradient bias. Third, we design a smooth curriculum scheduler that balances the two scores over training, enabling a lightweight transition from core-pattern consolidation to rare-factor exploration. The feature extractor is treated as a plug-in module: we use CLIP by default following prior practice~\cite{yang2024clip}, but downstream-model can also be used to train the sparse autoencoder with comparable performance, indicating the framework is not tied to a specific encoder.

\begin{figure*}
    \centering
    \includegraphics[width=1\linewidth]{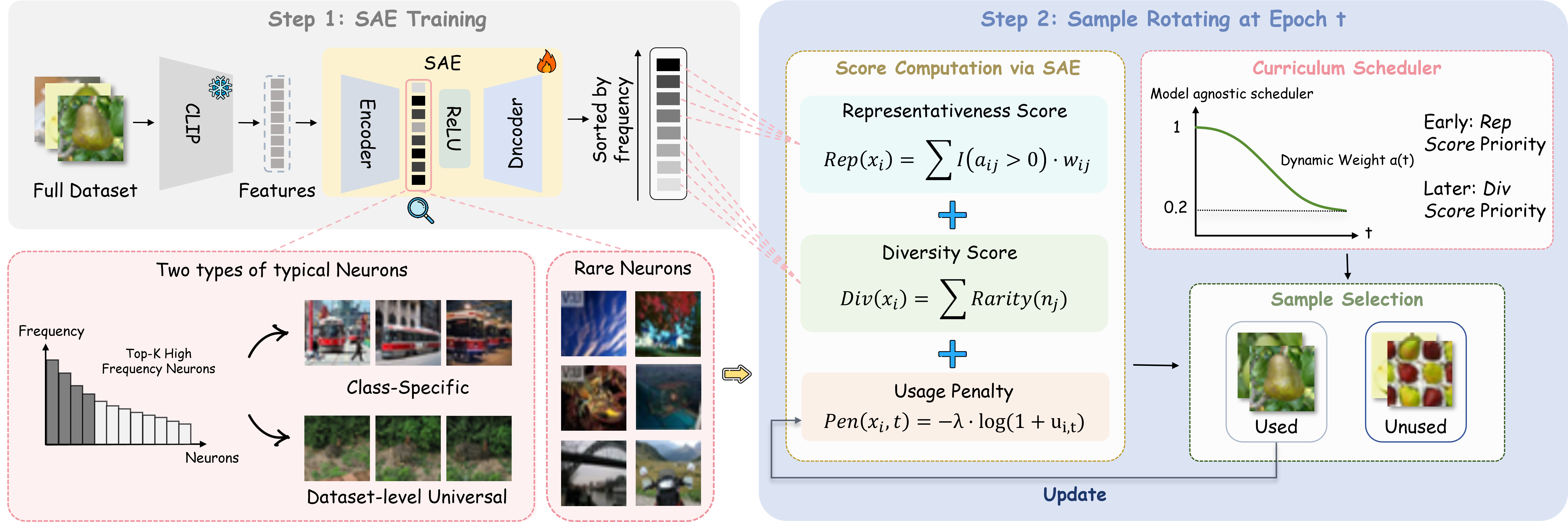}
    \caption{\textbf{Our framework for dynamic data selection.}
A plug-in feature encoder (CLIP by default) maps inputs to a fixed feature space, where an SAE yields sparse unit activations and dataset-wide statistics.
At epoch $t$, each example is scored by representativeness (weighted coverage of high-frequency factors), diversity (factor rarity), and a usage-frequency penalty that discourages repeated selection and enforces rotation.
A scheduler $\alpha(t)$ smoothly balances representativeness and diversity, transitioning from core-pattern consolidation to rare-feature exploration.}
    \label{fig:2}
\end{figure*}

Our framework is model-agnostic and applicable across modalities and tasks. We evaluate on five benchmarks, including CIFAR-10, CIFAR-100, Tiny-ImageNet, ImageNet-1K, and RSD\_15K, using diverse architectures such as ResNet-18/50, ViT, and VGG, and observe consistent improvements in accuracy–efficiency trade-offs.

% Our method is applicable to different modalities and tasks. We evaluate on five benchmarks—CIFAR-10/100, Tiny-ImageNet, ImageNet-1K, and RSD\_15K—using diverse architectures including ResNet-18/50, ViT, and VGG.
Our contributions can be summarized as follows:

\begin{itemize}
\item[$\bullet$] \textbf{Representativeness as high-frequency feature coverage.} We rethink representativeness as dataset-level coverage of common or high-frequency feature factors, rather than proximity to geometric centers, using sparse-unit activations in a plug-in feature space.
\item[$\bullet$] \textbf{Diversity as process-level rotation.} We rethink diversity as a training-process constraint by prioritizing complementary rare factors and enforcing rotation with a Usage-Frequency Penalty, reducing sample monopoly and its induced gradient bias, with a theoretical anti-monopoly guarantee on the ranking dynamics.
\item[$\bullet$] \textbf{Curriculum scheduler for balancing representativeness and diversity. } We introduce a lightweight curriculum scheduler that smoothly transitions the selection focus from high-frequency factor coverage (representativeness) to process-level diversity (via rotation) across epochs.
\end{itemize}

\section{Related Works}
\label{sec:related}

Data selection reduces training cost by identifying informative subsets while preserving model performance.
Prior work is typically categorized into static and dynamic strategies.

\textbf{Static selection} fixes a subset of samples before training.
Representative directions include coreset formulations and diversity-preserved heuristics that encourage coverage or dispersion in a representation space~\cite{pan2024gdiggradientbaseddiversehighquality,yang2025diversitydrivendataselectionlanguage,agarwal2020contextualdiversityactivelearning}.

\textbf{Dynamic selection} updates the subset along training.
Many dynamic methods rely on model-dependent feedback like loss across epochs~\cite{raju2021acceleratingdeeplearningdynamic,qin2024infobatch}.
These methods have achieved strong acceleration in practice and motivate continued study of principled scoring and scheduling.

Different from previous methods, our work starts by rethinking how representativeness and diversity are operationalized in dynamic selection.
Instead of using local geometric proxies for representativeness, we quantify it as weighted coverage of common, dataset-level factors in a plug-in feature space.
For diversity, we focus on the \emph{training process} by combining rare-factor sampling with a usage-frequency penalty that encourages sample rotation, which in turn helps alleviate long-term selection bias.
Due to space constraints, a more extensive discussion of related work is provided in the Apendix~\ref{sec:related2}.

\section{Method}
\label{sec:method}

\subsection{Problem Setup and Overview}
\label{subsec:setup}

Let $\mathcal{D}=\{d_i\}_{i=1}^{|\mathcal{D}|}$ be a labeled dataset, where $x_i$ is an input instance and $y_i$ its label.
We perform dynamic data selection once per epoch.
At epoch $t$, we construct a subset $\mathcal{S}(t)\subseteq\mathcal{D}$ with selection ratio $p$ by ranking examples using a time-dependent score $\mathcal{H}(d_i,t)$:
\begin{equation}
    \mathcal{S}(t)=\underset{\mathcal{S}\subseteq\mathcal{D}}{\arg\max}\ \sum_{d_i\in\mathcal{S}} \mathcal{H}(d_i,t)
    \quad \text{s.t.}\quad \frac{|\mathcal{S}|}{|\mathcal{D}|}=p.
    \label{eq:dynamical_selection}
\end{equation}
We rethink two core notions for scoring examples: representativeness and diversity.
Representativeness is defined as weighted coverage of dataset-level common/high-frequency feature factors, moving beyond local geometric centrality.
Diversity is treated as a process-level constraint that promotes complementary rare-factor exposure over epochs, rather than dispersion within a single subset.
We combine the two signals with a smooth curriculum scheduler to obtain $\mathcal{H}(d_i,t)$.

\subsection{Feature Factors via Sparse Units}
\label{subsec:semantic_probe}

Our scoring relies on a set of \emph{sparse units} that serve as latent feature factors for characterizing each example.
Given an input $x_i$, we denote by $a_{ij}$ the activation of unit $n_j$ and define the active set
\begin{equation}
    \mathcal{A}_i \triangleq \left\{ j \mid \mathbb{I}(a_{ij}>0)=1 \right\}.
    \label{eq:active_set}
\end{equation}
These units provide a compact factorization of the chosen feature space and support dataset-wide statistics used by our representativeness and diversity scores.

\paragraph{Instantiation.}
We treat the feature extractor as a plug-in module.
By default, we follow \cite{yang2024clip}, use CLIP encoder to extract features and train a sparse autoencoder on the target dataset to obtain sparse unit activations $\{a_{ij}\}$.
We refer to this component as a \emph{sparse-unit probe}, which maps representations from the chosen feature space to sparse activations; the remaining selection pipeline remains unchanged when swapping the feature extractor.
We also train the probe on downstream-model as an alternative instantiation, with results reported in the Appendix~\ref{sec:feature}.

\subsection{Rethinking Representativeness: Coverage over High-Frequency Factors}
\label{subsec:rep}

We define representativeness to reflect \emph{dataset-level coverage} of common/high-frequency feature factors, instead of local geometric centrality.
This is motivated by the observation that purely geometry-based coverage can preserve local structure while still under-covering global, implicitly distributed attributes.

Let $\mathcal{F}$ be the set of common factors, instantiated as the top-$K$ sparse units with the highest activation frequency over the full dataset.
A sample is representative if it activates many factors in $\mathcal{F}$.
To avoid over-valuing ubiquitous factors that appear broadly across classes, we weight each factor by the inverse of its class coverage.
Formally, for unit $n_j\in\mathcal{F}$, let $c_j$ be the number of distinct classes in which $n_j$ is activated at least once and define $w_j=1/c_j$.
We score representativeness as
\begin{equation}
    Rep(i) \triangleq \sum_{j\in\mathcal{F}} \mathbb{I}(a_{ij}>0)\cdot w_j.
    \label{eq:representativeness}
\end{equation}
This score can be viewed as weighted high-frequency coverage: it emphasizes covering common factors while discounting those shared across many classes.

\paragraph{Distributional faithfulness via MMD.}
To quantify how closely a selected subset matches the full-data distribution, we use Maximum Mean Discrepancy (MMD) in the same representation space as the sparse-unit probe (CLIP embeddings by default).
Let $\phi(\cdot)$ denote this fixed feature extractor and let $\mathcal{V}=\{\phi(x_i)\}_{i=1}^{N}$ be the full set of features.
For a selected subset $\mathcal{S}$ of size $K$, define $\mathcal{U}=\{\phi(x)\mid x\in\mathcal{S}\}$.
Given a characteristic kernel $k(\cdot,\cdot)$, the squared MMD is
\begin{equation}
\mathrm{MMD}^2(\mathcal{U},\mathcal{V})
=
\Bigl\|
\frac{1}{K}\sum_{u\in\mathcal{U}}\psi(u)
-
\frac{1}{N}\sum_{v\in\mathcal{V}}\psi(v)
\Bigr\|_{\mathcal{H}}^2,
\label{eq:mmd_def}
\end{equation}
where $\psi(\cdot)$ is the implicit feature map associated with $k$ and $\mathcal{H}$ is the corresponding RKHS.
We compute MMD using the standard unbiased estimator (see Appendix~\ref{subsec:mmd_appendix} for details).

\paragraph{Protocol.}
For each selection ratio $p$ (thus $K=\lfloor pN\rfloor$), we compare \textbf{Rep-TopK} (top-$K$ by $Rep(i)$) against a geometry-based baseline \textbf{K-Center} in the same feature space.
Figure~\ref{fig:mmd} reports the resulting MMD values across selection ratios.

\begin{figure}[t]
    \centering
    \includegraphics[width=1.0\linewidth]{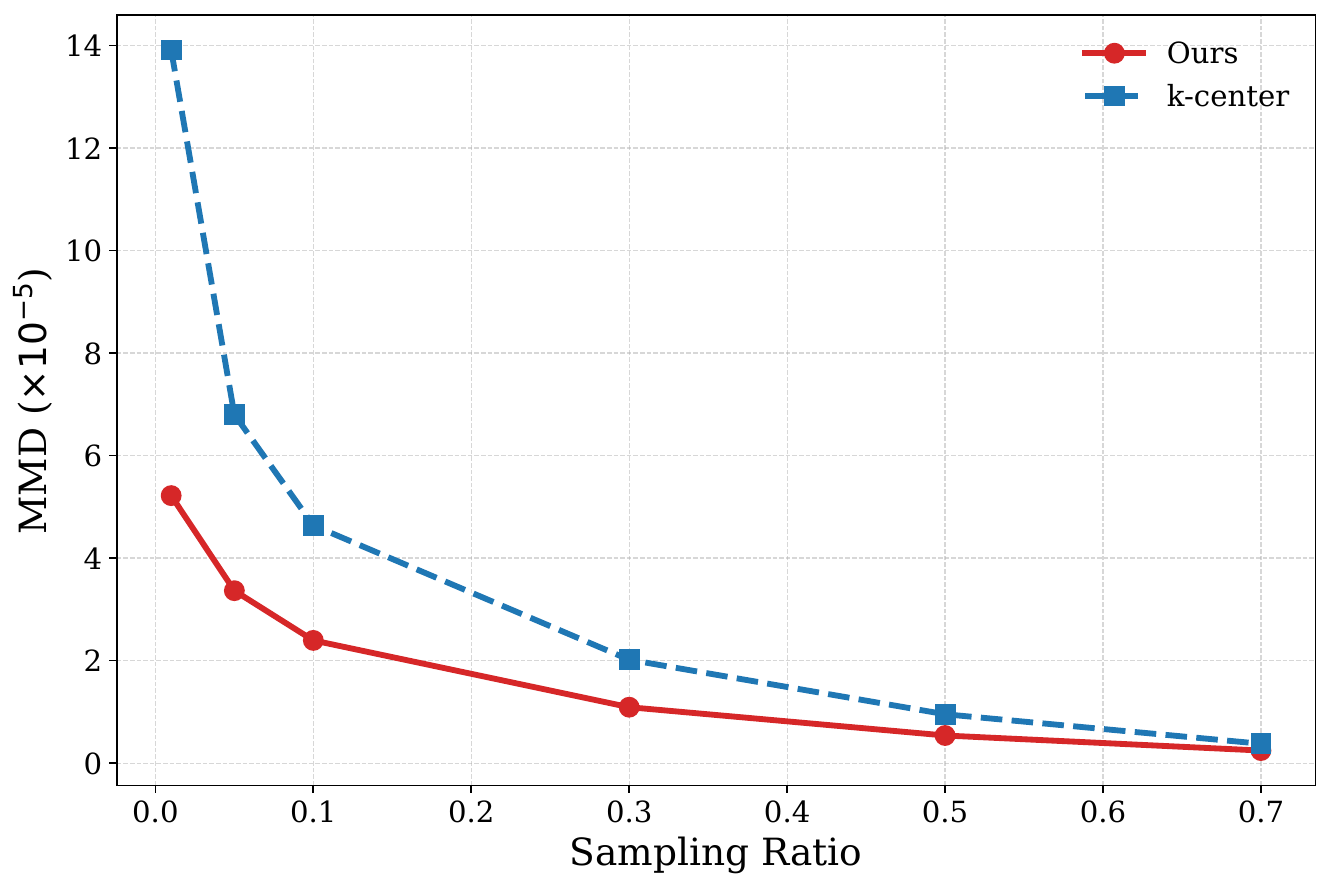}
    \caption{MMD comparison between subsets selected by our representativeness score (Rep-TopK) and the geometric baseline K-Center under different selection ratios. Lower is better.}
    \label{fig:mmd}
\end{figure}

\subsection{Rethinking Diversity: Process-Level Diversity}
\label{subsec:process_div}

Diversity is often defined within a single subset (e.g., geometric dispersion).
Under small selection ratios, a single subset can under-cover long-tail factors.
Moreover, in dynamic selection, a common failure mode is that a small set of high-scoring samples is repeatedly selected across epochs, resulting in poor factor coverage \emph{over time}.
We propose \textbf{Process-Level Diversity}, which encourages coverage throughout training by combining (i) rare-factor selection and (ii) explicit sample rotation.

\textbf{Rare-Factor Diversity.}
We measure the rarity of a sparse unit by the inverse of its dataset-level activation count:
\begin{equation}
    Rar(n_j) \triangleq \frac{1}{\sum_{i\in\mathcal{D}} \mathbb{I}(a_{ij}>0)}.
    \label{eq:rarity}
\end{equation}
A sample is diverse if it activates rare units. We define
\begin{equation}
    Div(i) \triangleq \sum_{j\in\mathcal{A}_i} Rar(n_j).
    \label{eq:diversity}
\end{equation}

\textbf{Usage-Frequency Penalty for Sample Rotation.}
We further discourage repeated selection via a usage-frequency penalty, which promotes rotation across epochs and mitigates sample monopoly.

\paragraph{Motivation: repeated selection induces gradient bias.}
Let $N=|\mathcal{D}|$ and $\mathcal{L}(\theta)=\frac{1}{N}\sum_{i=1}^{N}\ell_i(\theta)$ be the empirical risk.
The full-dataset gradient is
\begin{equation}
    g_{\mathrm{full}}(\theta)=\nabla \mathcal{L}(\theta)=\frac{1}{N}\sum_{i=1}^{N}\nabla \ell_i(\theta).
    \label{eq:full_grad}
\end{equation}
At epoch $t$, dynamic selection trains on a subset $\mathcal{D}_t\subseteq\mathcal{D}$ of size $K=\lfloor pN\rfloor$ with gradient
\begin{equation}
    g_t(\theta)=\frac{1}{K}\sum_{i\in \mathcal{D}_t}\nabla \ell_i(\theta).
    \label{eq:subset_grad}
\end{equation}
Define the long-term inclusion frequency of sample $i$ as
\begin{equation}
    \pi_i \triangleq \lim_{T\to\infty}\frac{1}{T}\sum_{t=1}^{T}\mathbb{I}\{i\in \mathcal{D}_t\}.
    \label{eq:pi_def}
\end{equation}
Then the expected subset gradient is a reweighted full gradient:
\begin{equation}
    \mathbb{E}[g_t(\theta)] = \sum_{i=1}^{N}\frac{\pi_i}{K}\nabla \ell_i(\theta),
    \label{eq:expected_subset_grad}
\end{equation}
and the gradient bias is
\begin{equation}
\begin{aligned}
    \Delta g(\theta)&\triangleq \mathbb{E}[g_t(\theta)]-g_{\mathrm{full}}(\theta) \\
    &=
    \sum_{i=1}^{N}\left(\frac{\pi_i}{K}-\frac{1}{N}\right)\nabla \ell_i(\theta).
    \label{eq:grad_bias}
\end{aligned}
\end{equation}
When a small set of samples dominates selection, $\{\pi_i\}$ becomes highly non-uniform, yielding systematic deviation from $g_{\mathrm{full}}(\theta)$.
Thus, reducing ``sample monopoly'' and encouraging more balanced inclusion frequencies (ideally $\pi_i\approx K/N$) is desirable for mitigating bias.

\paragraph{Penalty.}
Let $u_i(t)$ denote the cumulative number of times sample $i$ has been selected up to epoch $t$.
We define a sublinear log penalty
\begin{equation}
    Pen(i,t)=\lambda \log(1+u_i(t)),
    \label{eq:penalty}
\end{equation}
which is subtracted from the sampling score so that overused samples gradually lose priority, promoting rotation across training.

\paragraph{Anti-monopoly property.}
Define the \emph{base score} (before applying the usage penalty) as
\begin{equation}
    s_i(t)\triangleq \alpha(t)\,Rep(i)+\bigl(1-\alpha(t)\bigr)\,Div(i).
    \label{eq:base_score}
\end{equation}
Assume $s_i(t)\in[s_{\min},s_{\max}]$ and let $\Delta=s_{\max}-s_{\min}$.
Define the penalized score $\tilde{s}_i(t)=s_i(t)-\lambda\log(1+u_i(t))$.

\begin{theorem}[Sample Rotation / Anti-Monopoly]
\label{thm:anti_monopoly}
For any two samples $i$ and $j$, if
\begin{equation}
    u_i(t)\ge (1+u_j(t))\exp(\Delta/\lambda)-1,
    \label{eq:anti_monopoly_cond}
\end{equation}
then $\tilde{s}_i(t)\le \tilde{s}_j(t)$; i.e., sufficiently over-sampled instances cannot dominate the ranking indefinitely.
\end{theorem}

\begin{proof}
We have
\begin{equation}
\begin{aligned}
\tilde{s}_i(t)-\tilde{s}_j(t)
&=(s_i(t)-s_j(t))
-\lambda\log\!\frac{1+u_i(t)}{1+u_j(t)} \\
&\le
\Delta
-\lambda\log\!\frac{1+u_i(t)}{1+u_j(t)}.
\end{aligned}
\end{equation}
By the condition in~\eqref{eq:anti_monopoly_cond},
$\log\!\frac{1+u_i(t)}{1+u_j(t)} \ge \Delta/\lambda$,
which implies
$\tilde{s}_i(t)-\tilde{s}_j(t)\le 0$.\qedhere
\end{proof}

We evaluate the effect of sample rotation under noisy or hard examples in the Appendix~\ref{sec:noiseRubust}, where repeated selection can otherwise dominate the training signal.

\subsection{Curriculum Scheduling of Representativeness and Process-Level Diversity}
\label{subsec:scheduler}

We use a lightweight curriculum to trade off representativeness and process-level diversity over training.
At epoch $t$, the sampling score is
\begin{equation}
\mathcal{H}(i,t) \triangleq s_i(t) - Pen(i,t),
\label{eq:score_def}
\end{equation}
with a convex combination
\begin{equation}
s_i(t)
=
\alpha(t)\,Rep(i)
+
\bigl(1-\alpha(t)\bigr)\,Div(i).
\label{eq:base_score_sched}
\end{equation}
A larger $\alpha(t)$ early on emphasizes representative examples that cover high-frequency factors, while a smaller $\alpha(t)$ later increases rare-factor exposure over epochs through process-level diversity and rotation.
Importantly, this schedule is model-agnostic and does not rely on gradients or higher-order information. Optionally, we run a short full-data refinement stage in the last $T_{\mathrm{full}}$ epochs by setting $\mathcal{S}(t)=\mathcal{D}$, further reducing any residual bias. To mitigate class imbalance, we apply a simple class-balanced constraint by selecting an equal number of samples from each class at every epoch.

We adopt a smooth sigmoid schedule with a floor $\alpha_{\min}$:
\begin{equation}
\begin{aligned}
\alpha(t)
&=
\alpha_{\min}
+
\bigl(1-\alpha_{\min}\bigr)
\Bigl(1-\sigma\!\bigl(k(t-t_{\mathrm{mid}})\bigr)\Bigr), \\
\sigma(x)
&=
\frac{1}{1+e^{-x}}.
\end{aligned}
\label{eq:alpha}
\end{equation}
Here $t_{\mathrm{mid}}$ sets the transition point and $k$ controls the sharpness of the change. The overall procedure is summarized in Algorithm~\ref{alg:dds}.

\begin{algorithm}[t]
\caption{}
\label{alg:dds}
\begin{algorithmic}[1]
\REQUIRE Dataset $\mathcal{D}$; selection ratio $p$; epochs $T$; full-data refinement ratio $\rho$.
\REQUIRE Precomputed per-sample scores $Rep(i)$ and $Div(i)$.
\REQUIRE Scheduler $\alpha(t)$ in \cref{eq:alpha}, penalty weight $\lambda$.
\STATE Initialize $u_i(0)=0$ for all samples.

\FOR{$t=1$ to $T$}
    \IF{$t > (1-\rho)T$}
        \STATE Set $\mathcal{S}(t)=\mathcal{D}$ \hfill{\footnotesize (train on full data)}
    \ELSE
        \STATE Compute $\alpha(t)$ by \cref{eq:alpha}
        \FORALL{sample $i$}
            \STATE $Pen(i,t)=\lambda\log(1+u_i(t-1))$
            \STATE Update $\alpha(t)$
            \STATE Update $\mathcal{H}(i,t)$
        \ENDFOR
        \STATE Select $\mathcal{S}(t)$ as the top-$\lfloor p|\mathcal{D}|\rfloor$ samples by $\mathcal{H}(i,t)$
    \ENDIF

    \STATE Train the model for one epoch on $\mathcal{S}(t)$
    \STATE Update $u_i(t)=u_i(t-1)+\mathbb{I}\{i\in\mathcal{S}(t)\}$
\ENDFOR
\end{algorithmic}
\end{algorithm}

\section{Experiments}
\label{sec:exp}

\begin{table*}[!ht]
\centering
\caption{Accuracy (\%) comparison on CIFAR-10/100. Random$^*$ denotes dynamic random selection. The reported ratios denote the \textbf{selection ratios}. \textbf{Ours} is highlighted in grey. \ding{51} = Dynamic, \ding{55} = Static.}
\label{tab:cifar_30_70}
\setlength{\tabcolsep}{5pt}
\begin{tabular}{@{}c|l|c|cc|cc@{}}
\toprule
\multirow{2}{*}{\textbf{Backbone}} & \multirow{2}{*}{\textbf{Method}} & \multirow{2}{*}{\textbf{Type}} &
\multicolumn{2}{c|}{\textbf{CIFAR-10} (Selection Ratio)} & \multicolumn{2}{c}{\textbf{CIFAR-100} (Selection Ratio)} \\
& & & \textbf{70\%} & \textbf{30\%} & \textbf{70\%} & \textbf{30\%} \\
\midrule

\multirow{14}{*}{ResNet-18}
& Random & \ding{55}
& 94.6$_{\downarrow1.5}$ & 90.2$_{\downarrow5.9}$
& 73.8$_{\downarrow4.9}$ & 69.7$_{\downarrow9.0}$ \\
& GraNd-4 \cite{paul2023deeplearningdatadiet} & \ding{55}
& 95.3$_{\downarrow0.8}$ & 91.2$_{\downarrow4.9}$
& 74.6$_{\downarrow3.6}$ & 68.8$_{\downarrow9.4}$ \\
& MoDS \cite{xia2023moderate} & \ding{55}
& 93.9$_{\downarrow2.2}$ & 90.6$_{\downarrow5.5}$
& 74.6$_{\downarrow3.6}$ & 65.3$_{\downarrow12.9}$ \\
& MoSo \cite{Tan2023DataPV} & \ding{55}
& 95.3$_{\downarrow0.8}$ & 91.1$_{\downarrow6.0}$
& 77.5$_{\downarrow0.7}$ & 70.9$_{\downarrow7.3}$ \\
& $\mathbb{D}^2$ \cite{maharana2024mathbbd} & \ding{55}
& 95.7$_{\downarrow0.4}$ & 93.3$_{\downarrow2.8}$
& 78.2$_{\uparrow0.0}$ & 70.5$_{\downarrow7.7}$ \\
& DP \cite{yang2023dataset} & \ding{55}
& 94.9$_{\downarrow1.2}$ & 90.8$_{\downarrow5.3}$
& 77.2$_{\downarrow1.0}$ & -- \\
& Random$^*$ & \ding{51}
& 94.8$_{\downarrow1.3}$ & 93.0$_{\downarrow2.6}$
& 77.3$_{\downarrow0.9}$ & -- \\
& $\epsilon$-greedy \cite{raju2021acceleratingdeeplearningdynamic} & \ding{51}
& 95.2$_{\downarrow0.9}$ & 94.1$_{\downarrow2.0}$
& 76.4$_{\downarrow1.8}$ & -- \\
& UCB \cite{raju2021acceleratingdeeplearningdynamic} & \ding{51}
& 95.3$_{\downarrow0.8}$ & 93.9$_{\downarrow2.2}$
& 77.3$_{\downarrow0.9}$ & -- \\
& InfoBatch \cite{qin2024infobatch} & \ding{51}
& 95.6$_{\downarrow0.5}$ & 94.7$_{\downarrow1.4}$
& 78.2$_{\uparrow0.0}$ & 76.5$_{\downarrow1.7}$ \\
& RS2 w/r \cite{ICLR2024_68b8d2bc} & \ding{51}
& 95.3$_{\downarrow0.8}$ & 94.2$_{\downarrow1.9}$ 
& 77.1$_{\downarrow1.1}$ & 74.9$_{\downarrow3.3}$  \\
& RS2 w/o \cite{ICLR2024_68b8d2bc} & \ding{51}
& 94.9$_{\downarrow1.2}$ & 94.1$_{\downarrow2.0}$ 
& 76.9$_{\downarrow1.3}$ & 75.8$_{\downarrow2.4}$ \\
& RCAP \cite{hassan2025rcap}& \ding{51}
& 95.9$_{\downarrow0.2}$ & 94.4$_{\downarrow1.7}$ 
& 77.6$_{\downarrow0.7}$ & 75.6$_{\downarrow2.6}$ \\
& \cellcolor{gray!15}\textbf{Ours} & \cellcolor{gray!15}\ding{51}
& \cellcolor{gray!15}\textbf{96.1}$_{\uparrow0.0}$ & \cellcolor{gray!15}\textbf{95.3}$_{\downarrow0.8}$
& \cellcolor{gray!15}\textbf{78.4}$_{\uparrow0.2}$ & \cellcolor{gray!15}\textbf{76.7}$_{\downarrow1.5}$ \\
\cmidrule(lr){2-7}
& \textbf{Whole Dataset} & --
& \multicolumn{2}{c|}{96.1$_{\pm0.1}$}
& \multicolumn{2}{c}{78.2$_{\pm0.1}$} \\
\midrule

\multirow{9}{*}{ResNet-50}
& Random & \ding{55}
& 95.0$_{\downarrow0.8}$ & 90.3$_{\downarrow5.5}$
& -- & 53.6$_{\downarrow25.5}$ \\
& MoDS \cite{xia2023moderate} & \ding{55}
& 95.3$_{\downarrow0.5}$ & 91.1$_{\downarrow4.7}$
& -- & 57.8$_{\downarrow21.3}$ \\
& Self-sup. \cite{NEURIPS2022_7b75da9b} & \ding{55}
& 95.0$_{\downarrow0.8}$ & 90.1$_{\downarrow5.7}$
& -- & 54.6$_{\downarrow24.5}$ \\
& CLIP \cite{yang2024clip} & \ding{55}
& 95.4$_{\downarrow0.4}$ & 91.1$_{\downarrow4.7}$
& -- & 62.3$_{\downarrow12.8}$ \\
& RS2 w/r \cite{ICLR2024_68b8d2bc} & \ding{51}
& 95.4$_{\downarrow0.4}$ & 94.0$_{\downarrow1.8}$ 
& 78.4$_{\downarrow0.7}$ & 75.1$_{\downarrow3.3}$  \\
& RS2 w/o \cite{ICLR2024_68b8d2bc} & \ding{51}
& 95.3$_{\downarrow0.5}$ & 93.1$_{\downarrow2.7}$ 
& 75.3$_{\downarrow3.8}$ & 72.7$_{\downarrow6.4}$ \\
& RCAP \cite{hassan2025rcap}& \ding{51}
& 95.2$_{\downarrow0.6}$ & 94.4$_{\downarrow1.4}$ 
& 78.0$_{\downarrow1.1}$ & 74.2$_{\downarrow4.9}$ \\
& \cellcolor{gray!15}\textbf{Ours} & \cellcolor{gray!15}\ding{51}
& \cellcolor{gray!15}\textbf{95.6}$_{\downarrow0.2}$ & \cellcolor{gray!15}\textbf{94.5}$_{\downarrow1.3}$
& \cellcolor{gray!15}\textbf{79.1}$_{\uparrow0.0}$ & \cellcolor{gray!15}\textbf{75.7}$_{\downarrow3.4}$ \\
\cmidrule(lr){2-7}
& \textbf{Whole Dataset} & --
& \multicolumn{2}{c|}{95.8$_{\pm0.1}$}
& \multicolumn{2}{c}{79.1$_{\pm0.1}$} \\
\bottomrule
\end{tabular}
\end{table*}

\subsection{Experiment Setup}

\noindent\textbf{Datasets and backbones.}
We evaluate our dynamic data selection framework on four image-classification benchmarks—CIFAR-10/100~\cite{Krizhevsky2009LearningML}, Tiny-ImageNet~\cite{chrabaszcz2017downsampledvariantimagenetalternative}, and ImageNet-1K—and one text-classification benchmark, RSD\_15K~\cite{Zheng2025RSD15KAL}.
We consider diverse architectures, including ResNet-18/34/50~\cite{he2015deepresiduallearningimage}, VGG~\cite{Simonyan2014VeryDC}, and ViT~\cite{dosovitskiy2021an} for vision, and RoBERTa \cite{Liu2019RoBERTaAR} for RSD\_15K \cite{Zheng2025RSD15KAL}.
CIFAR-10/100 results are summarized in \cref{tab:cifar_30_70}, while ImageNet-1K and RSD\_15K results are reported in \cref{tab:imagenet_tiny_singlecol}.
For Tiny-ImageNet, we provide the complete training details and results in Appendix~\ref{app:tiny} for reproducibility.

\noindent\textbf{Dynamic selection protocol.}
Unless stated otherwise, the reported ratio denotes the \textbf{per-epoch selection ratio}: at the beginning of each epoch, we re-sample a subset and train on the selected examples only.
For all comparisons, we keep the training schedule and optimization hyperparameters identical across methods and only change the data stream induced by each selection strategy.
Unless noted, we further use full-data training for the last 15\% epochs as a short refinement stage.
Additional implementation details are provided in Appendix~\ref{sec:implement}.

\subsection{Comparison with the State-of-the-Arts}

\noindent\textbf{Setup.}
We compare against representative \emph{static} and \emph{dynamic} data selection baselines on CIFAR-10/100 under two per-epoch \textbf{selection ratios} (30\% and 70\%). We include commonly used static baselines based on proxy or feature-space criteria (GraNd-4~\cite{paul2023deeplearningdatadiet}, MoDS~\cite{xia2023moderate}, MoSo~\cite{Tan2023DataPV}, $\mathbb{D}^2$~\cite{maharana2024mathbbd}, DP~\cite{yang2023dataset}) and dynamic baselines that update the data stream across epochs ($\epsilon$-greedy/UCB~\cite{raju2021acceleratingdeeplearningdynamic}, InfoBatch~\cite{qin2024infobatch}, RS2~\cite{ICLR2024_68b8d2bc}, RCAP~\cite{hassan2025rcap}).
For each method, we follow a consistent training protocol and compare under the settings that are supported by our implementation.

\noindent\textbf{Results.}
As shown in Table~\ref{tab:cifar_30_70}, our method remains competitive across datasets and selection ratios, and the results reveal clear patterns.
First, dynamic selection is inherently more robust at low selection ratios: even \emph{dynamic random} surpasses several static baselines at 30\%, underscoring that epoch-wise re-sampling alone already alleviates the brittleness of committing to a single fixed subset.
Second, beyond this “dynamic” advantage, our method delivers additional and consistent improvements over strong dynamic baselines at both 30\% and 70\%, with the largest gains on the more challenging CIFAR-100 benchmark. This is consistent with our design: early representativeness helps maintain a faithful proxy of the full data, while later rotation and rare-factor coverage prevent sample monopoly and broaden training signals over time.
Notably, among all compared methods, our approach most consistently matches full-data training performance across both ResNet-18 and ResNet-50 on CIFAR-10/100, indicating that aggressive data reduction need not sacrifice accuracy when selection is properly structured.

Results on ImageNet-1K and RSD\_15K in Table~\ref{tab:imagenet_tiny_singlecol} further support the generality of our scoring-and-rotation design beyond CIFAR.
On ImageNet-1K, we evaluate with ResNet-34 and control for compute by keeping the total number of forward passes identical to full-data training (increasing epochs proportionally as the selection ratio decreases), which isolates the effect of curriculum-based sampling from simple under-training.
On RSD\_15K, we additionally report a static variant that selects instances with high representativeness and diversity scores, showing that the two scores remain informative even without online rotation, while our full method further benefits from process-level rotation when dynamic selection is enabled.

\subsection{Efficiency Comparison.}

\textbf{Convergence speed.}
On CIFAR-10 with ResNet-18, we report convergence curves in \cref{fig:conv} and measure the steps to reach \textbf{90\%} validation accuracy (\emph{step@90}). In the early stage, our curve closely matches full-data training, supporting that the representativeness-dominant selection provides a faithful proxy of the full set. In the mid-to-late stage, our method improves faster and overtakes full-data training, suggesting that sample rotation and rare-factor coverage accelerate learning by exposing complementary signals over epochs. Overall, our method reaches the target accuracy in the fewest steps.

\begin{table}[t]
\centering
\caption{\textbf{Top-1 accuracy (\%)} on ImageNet-1K under different \textbf{selection ratios}.
Random$^{*}$ denotes dynamic random selection per epoch.}
\label{tab:imagenet_tiny_singlecol}
\setlength{\tabcolsep}{4pt}
\renewcommand{\arraystretch}{1.08}
\scriptsize
\resizebox{\columnwidth}{!}{
\begin{tabular}{l|l|cccc}
\toprule
\textbf{Dataset} & \textbf{Method} & \textbf{100\%} & \textbf{70\%} & \textbf{50\%} & \textbf{30\%} \\
\midrule

\multirow{7}{*}{\shortstack{\textbf{ImageNet-1K}\\{(ResNet-34)}}}
& Random      & 73.1 & 72.2 & 70.3 & 66.7 \\
& Entropy     & 73.1 &  72.3 & 70.8 & 64.1 \\
& $\mathbb{D}^2$     & 73.1 & 73.0 & 71.4 &  64.8 \\
& Moderate-DS & 73.1 & 72.0 & 70.3 & 65.9 \\
& K-center        & 73.1 & 72.2 & 67.2 & 48.8 \\
& Forgetting  & 73.1 & 72.6 & 70.9 & 66.5 \\
& \cellcolor{gray!10}\textbf{Ours}
& \cellcolor{gray!10}73.1
& \cellcolor{gray!10}\textbf{73.4}
& \cellcolor{gray!10}\textbf{73.7}
& \cellcolor{gray!10}\textbf{73.8} \\
\midrule
\multirow{2}{*}{\shortstack{\textbf{RSD\_15K}\\{(RoBERTa)}}}
& K-Center  & 71.0 & 71.6 & 69.7 & 69.1 \\
& \cellcolor{gray!10}\textbf{Ours}
& \cellcolor{gray!10}71.0
& \cellcolor{gray!10}\textbf{72.5}
& \cellcolor{gray!10}\textbf{71.5}
& \cellcolor{gray!10}\textbf{69.5} \\
\bottomrule
\end{tabular}
}
\end{table}

\begin{table}[]
    \centering
        \caption{The accuracy(\%) of different model architectures on CIFAR10/100 datasets. The model was trained using our method with 70\% samples selected per epoch. Our proposed method outperforms the full dataset on various occasions.}
    \begin{tabular}{cc|cc|ccc}
    \toprule
         & \textbf{Dataset} & \multicolumn{2}{c|}{CIFAR10} & \multicolumn{3}{c}{CIFAR100} \\
         \toprule
         & Model & R-18 & R-50 & R-18 & R-50 & VGG-16\\
         \toprule
         & Full & 96.1 & \textbf{95.8} & 78.2 & 79.1 & 73.9\\
         \toprule
         & Ours & \textbf{96.1} & 95.6 & \textbf{78.4} & \textbf{79.1} & \textbf{74.0}\\
         \toprule
    \end{tabular}
    \label{tab:overall_performance}
\end{table}

\textbf{Additional overhead.}
Our pipeline separates \emph{offline} scoring from \emph{online} updates.
After training the SAE once, we precompute Rep/Div scores for all samples; during training, we only update the usage-frequency penalty and the scheduler weight per epoch.
The SAE cost is thus a one-time, reusable component, and it transfers across datasets: as shown in \cref{tab:SAE_gene}, an SAE trained on ImageNet-1K yields comparable performance when used to score CIFAR-10.
Consequently, in the reusable-SAE setting, the amortized overhead is dominated by lightweight online updates.
\cref{tab:efficiency} summarizes the resulting speed--accuracy trade-off.

\begin{table}[t]
\centering
\caption{Top-1 accuracy (\%) on CIFAR-10 using ResNet-18. Models are trained with varying pruning ratios per epoch. I-1K $\to$ C-10 denotes ImageNet-1K trained SAE on CIFAR-10 scoring, then train ResNet-18 by our framework; C-10 $\to$ C-10 denotes our proposed method.}
\label{tab:SAE_gene}
\begin{tabular}{lccccc}
\toprule
\multirow{2}{*}{\textbf{Training Protocol}} & \multicolumn{5}{c}{\textbf{Selection Ratio per Epoch (\%)}} \\
\cmidrule(lr){2-6}
& 70 & 50 & 30 & Full \\
\midrule
I-1K $\to$ C-10 & 96.1 & 95.6 & \textbf{95.7} & 96.1 \\
C-10 $\to$ C-10 & 96.1 & 95.6 & 95.3 & 96.1 \\
\bottomrule
\end{tabular}
\end{table}

\begin{figure}[h]
    \centering
    \includegraphics[width=1\linewidth]{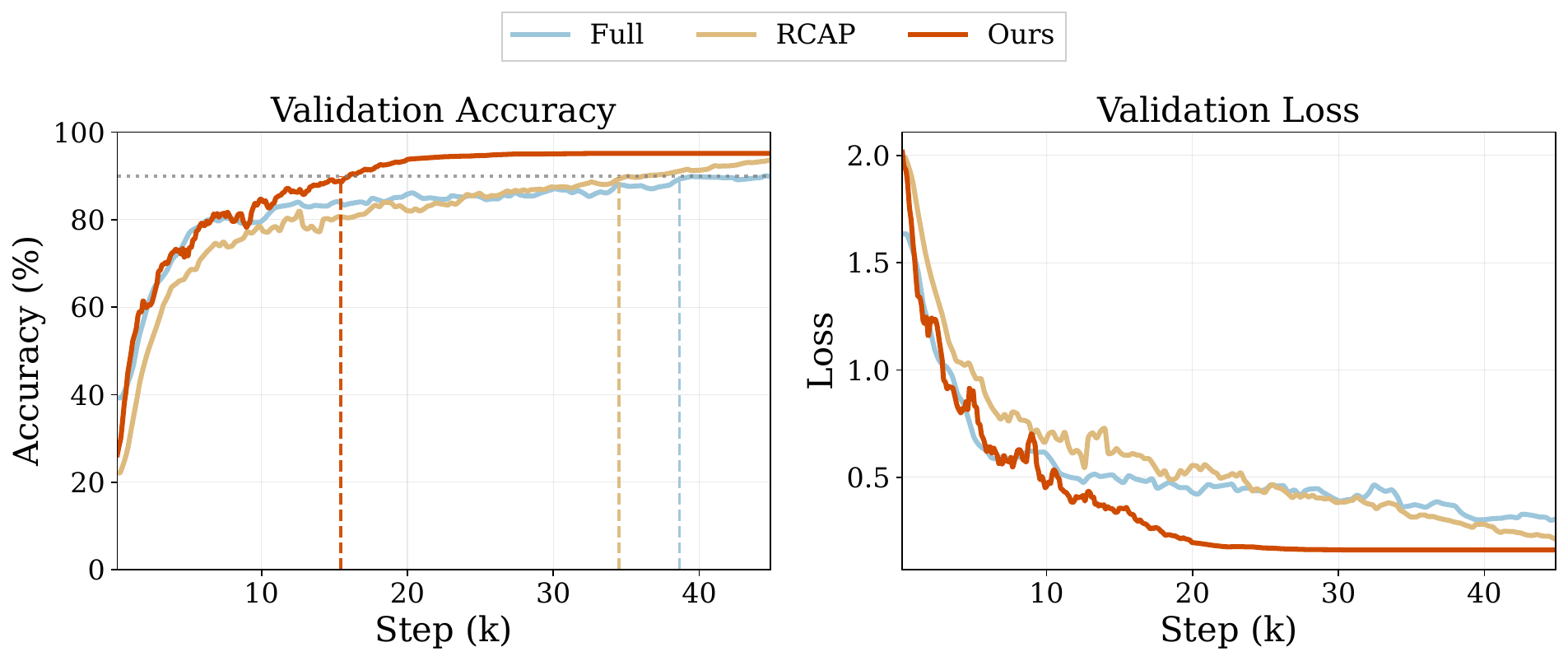}
    \caption{Convergence speed comparison on CIFAR-10 (ResNet-18). \textbf{Ours} denotes our method, \textbf{RCAP} denotes the previous SOTA method, and \textbf{Full} uses the full dataset.}
    \label{fig:conv}
\end{figure}

\subsection{Ablation Study}

We ablate each design choice on CIFAR-100 with ResNet-50 (200 epochs).
\cref{tab:ablation} reports accuracy at 50 epochs (optimization progress) and the final accuracy.

\textbf{Common-factor coverage supports early progress.}
Removing $Rep$ (\emph{w/o Rep}) consistently lowers both the 50-epoch and final accuracy, indicating that prioritizing coverage of common factors improves the early optimization trajectory under selection.

\textbf{Rare-factor exposure over time matters for the end performance.}
Removing $Div$ (\emph{w/o Div}) yields a relatively small gap at 50 epochs but a larger gap at the end, suggesting that process-level exposure to complementary rare factors is important for the later-stage performance.

\textbf{Rotation via usage-frequency penalty is critical.}
The usage-frequency penalty is the key mechanism that enforces sample rotation.
Disabling it (\emph{w/o Pen}) leads to a marked degradation, consistent with the observation that greedy re-selection can concentrate training on a small subset across epochs.
Interestingly, the penalty alone (\emph{only $Pen$}) already provides a strong baseline: by explicitly discouraging repeated inclusion, it spreads selection over the dataset and partially mitigates long-term sampling imbalance even without $Rep$ or $Div$.

\textbf{Scheduling improves the trade-off beyond a fixed mixture.}
Without the curriculum schedule (\emph{w/o $\alpha$}), performance drops compared to the full method.
A static mixing variant (\emph{Static})---fixed $Rep/Div$ combination and no penalty---performs substantially worse, showing the benefit of coupling time-varying emphasis with rotation.

\textbf{Full-data refinement stabilizes the final model.}
Removing the final full-data stage (\emph{w/o Full-data}) decreases the final accuracy, suggesting that a short refinement phase on the full dataset helps reduce residual selection-induced bias and consolidates the learned representation.
 
\begin{table}[t]
\centering
\caption{Training efficiency and accuracy comparison. Overhead includes CLIP feature extraction, SAE training, and samples scoring. Speedup = Full / Ours training time.}
\label{tab:efficiency}
\setlength{\tabcolsep}{6pt}
\begin{tabular}{@{}l c c c c@{}}
\toprule
\multirow{2}{*}{\textbf{Dataset}} &  \multirow{2}{*}{\textbf{Method}} & \multirow{2}{*}{\textbf{Acc (\%)}}  & \textbf{Overhead} & \multirow{2}{*}{\textbf{Speedup}} \\
& & & \textbf{(GPU h)}  & \\
\midrule
\multirow{2}{*}{C-10} & Full & 96.1\textsubscript{±0.1}  & – & 1.0× \\
         & \textbf{Ours} & \textbf{95.7}\textsubscript{±0.1} & \textbf{0.0028} & \textbf{2.51×} \\
\midrule
\multirow{2}{*}{T-Img} & Full & 58.2\textsubscript{±0.1}  & – & 1.0× \\
            & \textbf{Ours} & 58.1\textsubscript{±0.2} & \textbf{0.12} & \textbf{1.34×} \\
\bottomrule
\end{tabular}
\end{table}

\begin{table}[t]
\centering
\caption{\textbf{Ablation on CIFAR-100 (ResNet-50, 200 epochs).}
$R$, $D$, and $P$ denote the representativeness score, process-level diversity score, and usage-frequency penalty, respectively; $\alpha$ is the curriculum scheduler.
\emph{w/o Full-data} removes the final full-data refinement stage.
\emph{only $Pen$} selects samples by rotation only, i.e., prioritizing those with lower historical usage.}
\label{tab:ablation}
\setlength{\tabcolsep}{6pt}
\renewcommand{\arraystretch}{1.05}
\begin{tabular}{lcccccc}
\toprule
Variant & $R$ & $D$ & $P$ & $ \alpha $ & 50 epochs & Final \\
\midrule
Full Method & \checkmark & \checkmark & \checkmark & \checkmark & \textbf{57.7} & \textbf{79.1} \\
\midrule
w/o $Rep$ & \ding{55} & \checkmark & \checkmark & \checkmark & 57.1 & 78.1 \\
w/o $Div$ & \checkmark & \ding{55} & \checkmark & \checkmark & 57.5 & 77.2 \\
w/o $Pen$ & \checkmark & \checkmark & \ding{55} & \checkmark & 54.7 & 76.4 \\
w/o $\alpha$ & \checkmark & \checkmark & \checkmark & \ding{55} & 54.5 & 78.2 \\
Static & \checkmark & \checkmark & \ding{55} & \ding{55} & 53.8 & 76.3 \\
only $Div$ & \ding{55} & \checkmark & \ding{55} & \ding{55} & 57.1 & 77.1 \\
only $Rep$ & \checkmark & \ding{55} & \ding{55} & \ding{55} & 54.1 & 76.8 \\
only $Pen$ & \ding{55} & \ding{55} & \checkmark & \ding{55} & 54.6 & 78.1 \\
w/o Full-data & \checkmark & \checkmark & \checkmark & \checkmark & 57.6 & 78.3 \\
\bottomrule
\end{tabular}
\end{table}

\subsection{Transferability Across Architectures and Modalities}
\label{subsec:transfer}

A key property of our framework is that the scoring module is \emph{model-agnostic} to the downstream learner.
Representativeness and diversity are computed once in a plug-in feature space with an SAE-based sparse-unit probe, and thus do not rely on architecture-specific training signals such as losses or gradients.
During training, only the usage-frequency penalty and the scheduler weight are updated online.

\textbf{Cross-backbone transfer.}
We evaluate the same fixed scoring module across diverse backbones, including CNNs and vision transformer.
As shown in \cref{tab:overall_performance,fig:vgg16}, our method consistently yields competitive accuracy--efficiency trade-offs under the same selection ratio, closely matching full-data training across architectures.
This suggests that emphasizing \emph{dataset-level coverage of common factors} together with process-level rotation provides stable selection signals that transfer across model families.

\textbf{Cross-modality transfer.}
Beyond vision classification, we additionally report results on a text classification benchmark (RSD\_15K in \cref{tab:imagenet_tiny_singlecol}).
The same scoring-and-rotation principle remains effective when an appropriate feature extractor is available, indicating that the framework extends beyond a specific vision setting.

\begin{figure}
    \centering
    \includegraphics[width=1\linewidth]{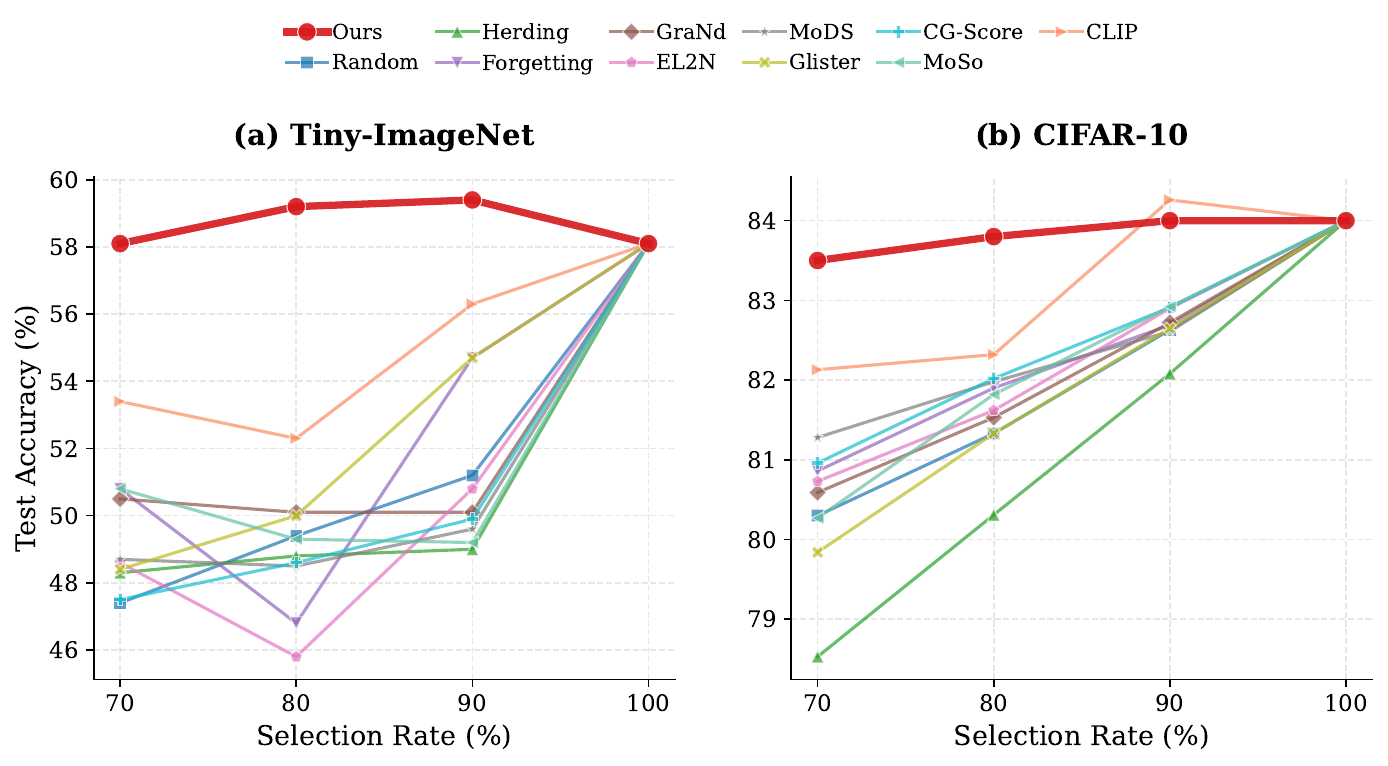}
    \caption{ Illustrations of comparing our method with various data selection baselines on different datasets and models. We trained VGG-16 on Tiny-ImageNet in Fig (a), ViT-small on CIFAR-10 in Fig (b).}
    \label{fig:vgg16}
\end{figure}

\textbf{Robustness to learning-rate schedules.}
We evaluate our method under three common schedulers while keeping other training settings fixed.
Across all schedules, our method preserves accuracy and delivers consistent training-time reductions; see Appendix~\ref{app:scheduler} for full settings and results.

\section{Conclusion}
\label{sec:conclusion}

This work rethinks dynamic data selection through two complementary lenses: representativeness as dataset-level coverage of high-frequency factors, and diversity as a process-level requirement enforced over epochs.
Using a fixed sparse-unit probe in a plug-in feature space, we precompute offline scores for representativeness and diversity, and update selection online only through a usage-frequency penalty and a lightweight scheduler, which together encourage sample rotation and reduce long-term selection bias.
Across multiple vision and text benchmarks and a range of backbones, the resulting framework yields a stronger speed--accuracy trade-off.
We hope our protocol and analyses will facilitate further study of scalable dynamic selection, and we plan to explore broader tasks and larger-scale settings and explore the synergistic effects with methods such as data synthesis and data augmentation in future work.

\section*{Impact Statement}
This paper presents work whose goal is to advance the field of machine learning. There are many potential societal consequences of our work, none of which we feel must be specifically highlighted here.
\bibliography{main}
\bibliographystyle{icml2026}
\newpage
\appendix
\onecolumn

\section{Training and Analysis of Sparse Autoencoder}

\subsection{Motivation}
Modern vision-language models (e.g., CLIP) provide strong but dense embeddings, where multiple factors are mixed across dimensions. 
For our selection scores, we therefore introduce a sparse factorization that turns dense features into sparse unit activations, enabling simple dataset-level statistics over \emph{common} versus \emph{rare} factors.

Specifically, we train a Top-$k$ Sparse Autoencoder (SAE) on the chosen embedding space (CLIP by default). 
The SAE produces an overcomplete and sparse latent code, where each sample activates only a small set of units. 
These sparse activations serve as our sparse-unit probe, from which we compute weighted coverage of common factors for representative and rarity-based signals for process-level diversity.

\subsection{Model Architecture}
Let $x \in \mathbb{R}^D$ be the CLIP-ViT/L-14 embedding of an image. We define a two-layer SAE with tied weights and Top-K sparsity:
\begin{equation}
\mathbf{z}_{\text{pre}} = \mathbf{W}_e \mathbf{x} + \mathbf{b}_e, \quad \mathbf{W}e \in \mathbb{R}^{N_{\text{lat}} \times D}, 
\end{equation}
\begin{equation}
\mathbf{z} = \text{TopK}(\mathbf{z}_{\text{pre}}, k), \
\hat{\mathbf{x}} = \mathbf{W}_d \mathbf{z} + \mathbf{b}_d, \quad 
\end{equation}
where $N_\text{lat} \gg D$ denotes latent dimension, $\mathbf{W}_d = \mathbf{W}_e^\top$, $k$ denotes sparsity budget, $TopK(z,k)$ retains only the $k$ largest entries and zeros out the rest:
\begin{equation}
    \mathrm{TopK}(\mathbf{z}, k)_i =
    \begin{cases}
        z_i, & \text{if } i \in \mathcal{I}_{\text{top-}k}(\mathbf{z}), \\
        0,   & \text{otherwise},
    \end{cases}
    \label{eq:topk}
\end{equation}
with $\mathcal{I}_{\text{top-}k}(\mathbf{z})$ being the index set of the $k$ largest entries in $\mathbf{z}$ by absolute value.
\subsection{Training Details}
The SAE is trained by minimizing the reconstruction loss, augmented with a \textit{dead neuron revival} term to ensure full utilization of the latent space. Let $\mathbf{z} = \mathrm{TopK}(\mathbf{W}_e\mathbf{x} + \mathbf{b}_e, k)$ denote the sparse activation. A neuron $j$ is considered \emph{dead} if its activation count remains zero over a monitoring window. Every $T_{\text{check}}$ batches, we identify the set of dead neurons $\mathcal{D} = \{ j \mid \sum_{i=1}^{B} \mathbb{I}(z_{i,j} > 0) = 0 \}$.

For each $j \in \mathcal{D}$, we construct a \emph{revival sample} $\mathbf{x}^{\text{rev}}_j$ as the training instance that maximally excites neuron $j$ under its current pre-activation:
\begin{equation}
    \mathbf{x}^{\text{rev}}_j = \arg\max_{\mathbf{x}_i \in \mathcal{B}} \left[ (\mathbf{W}_e \mathbf{x}_i + \mathbf{b}_e)_j \right],
    \label{eq:revival_sample}
\end{equation}
where $\mathcal{B}$ is the current mini-batch.

We then reconstruct $\mathbf{x}^{\text{rev}}_j$ using \emph{only} neuron $j$, i.e., define a fake activation vector $\tilde{\mathbf{z}}^{(j)}$ as:
\begin{equation}
    \tilde{\mathbf{z}}^{(j)}_l = 
    \begin{cases}
        (\mathbf{W}_e \mathbf{x}^{\text{rev}}_j + \mathbf{b}_e)_j, & l = j, \\
        0, & l \neq j,
    \end{cases}
    \label{eq:fake_sparse}
\end{equation}

and compute its reconstruction $\hat{\mathbf{x}}^{\text{from dead}}_j = \mathbf{W}_d \tilde{\mathbf{z}}^{(j)} + \mathbf{b}_d$.

The revival loss is then summed over all dead neurons in $\mathcal{D}$:
\begin{equation}
    \mathcal{L}_{\text{revive}} = \sum_{j \in \mathcal{D}} \left\| \mathbf{x}^{\text{rev}}_j - \hat{\mathbf{x}}^{\text{from dead}}_j \right\|_2^2,
    \label{eq:revive_loss}
\end{equation}
and the total training objective is:
\begin{equation}
    \mathcal{L} = \underbrace{\|\mathbf{x} - \hat{\mathbf{x}}\|_2^2}_{\text{Reconstruction Loss}}
    \;+\;
    \lambda_{\text{revive}} \cdot \mathcal{L}_{\text{revive}}.
    \label{eq:total_loss}
\end{equation}

This procedure encourages dormant neurons to specialize on under-represented patterns, significantly improving latent coverage and downstream curriculum quality.
\subsection{Visualization of our trained SAE}
To validate the semantic disentanglement of our SAE, we perform a multi-level visualization of the latent space.

\begin{figure}
    \centering
    \includegraphics[width=1\linewidth]{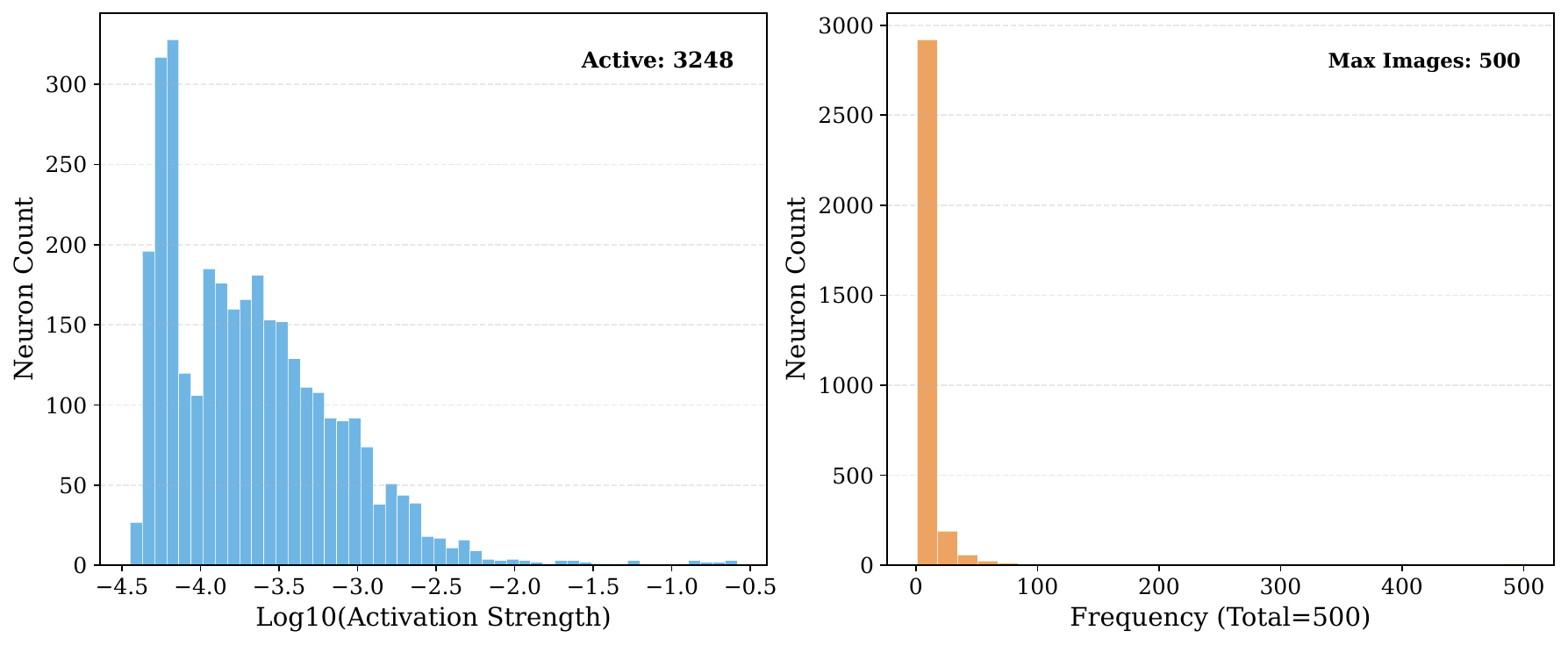}
    \caption{SAE analysis of the class 'cat'. The left panel shows the activation intensity of each neuron; selecting nearly the top 10\% of neurons as high-frequency neurons covers all neurons with high activation intensity. The right panel displays the number of samples that activate each neuron, where a maximum value of 500 indicates that all images in this class activate at least one neuron.}
    \label{fig:cat}
\end{figure}
\label{sec:rationale}

First, we project the entire latent representation using UMAP onto a 2D plane in \cref{fig:umap1}. The resulting plot reveals that samples from the same class form compact, well-separated clusters, while semantically related classes (e.g., car and bus) lie in proximity, and dissimilar classes (e.g., tulip and bicycle) are distant, as shown in \cref{fig:umap2}. This geometric structure confirms that the SAE has learned a hierarchically organized semantic space, where Euclidean distance aligns with human-perceived semantic similarity — a prerequisite for meaningful sample scoring.

Second, to examine the internal structure of a single class, we analyze the activation patterns for the 'cat' category in \cref{fig:cat}. We find that 5768 neurons are activated by at least one cat image, and the distribution of average activation strength is long-tailed, indicating that only a small subset of neurons are strongly activated — these likely correspond to core, discriminative features (e.g., “pointy ears”, “whiskers”). Moreover, the activation frequency is highly skewed: a few neurons are activated by all 5000 cat images, while most are active in only a few samples — suggesting the presence of both universal and rare, fine-grained visual concepts within the class.

\begin{figure}
    \centering
    \includegraphics[width=0.5\linewidth]{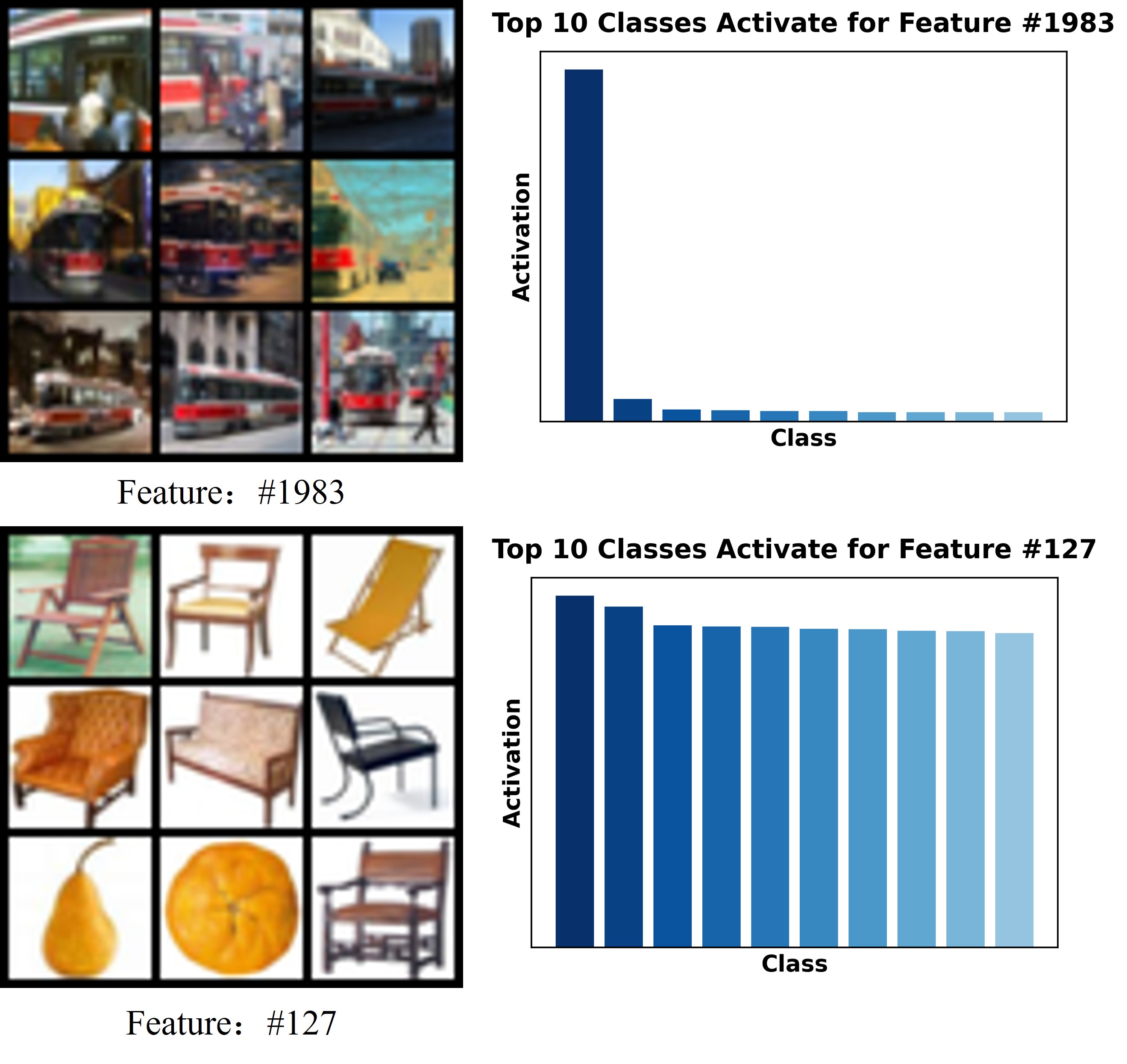}
    \caption{Visualization of two types of highly activated neurons and the nine images (on the left column) that most strongly activate them. The top row shows a class-specific neuron, which is strongly activated only in the class of street car. The bottom row shows a universal neuron, which is strongly activated across multiple classes and images with a plain     
    background.}
    \label{fig:neuron}
\end{figure}

% \begin{wrapfigure}{r}{0.42\linewidth}
%   \vspace{-6pt}
%   \centering
%   \includegraphics[width=\linewidth]{fig/neurons.jpg}
%   \caption{Visualization of two types of highly activated neurons and the nine images (on the left column) that most strongly activate them. The top row shows a class-specific neuron, which is strongly activated only in the class of street car. The bottom row shows a universal neuron, which is strongly activated across multiple classes and images with a plain     
%     background.}
%   \label{{fig:neuron}
%   \vspace{-10pt}
% \end{wrapfigure}

As shown in \cref{fig:neuron}, high frequency neurons can be divided into two categories: class-specific and universal. In our methods, we take the universal ones into consideration, which are considered to represent the implicit feature of the full dataset.

\section{Method Details}

\subsection{Robustness to Feature-Space Instantiation}
\label{sec:feature}
By default, we train the sparse-unit probe on CLIP embeddings. We additionally instantiate the probe on downstream ResNet-18 representations while keeping the rest of the selection pipeline unchanged; here the ResNet-18 encoder is pretrained on CIFAR-10 to provide a reasonable feature space.
As shown in Table~\ref{tab:instantiation_resnet18}, CLIP features perform best, while the downstream-feature probe is consistently lower by \textbf{0.2} accuracy points at 30\%, 50\%, and 70\% selection ratios, indicating the method is not tied to a specific feature space.

\begin{table}[t]
\centering
\caption{Effect of sparse-unit probe instantiation on ResNet-18.
We compare probes trained on CLIP features and downstream ResNet-18 representations under different selection ratios.}
\label{tab:instantiation_resnet18}
\setlength{\tabcolsep}{10pt}
\begin{tabular}{lccc}
\toprule
\textbf{Probe Feature Space} & \textbf{30\%} & \textbf{50\%} & \textbf{70\%} \\
\midrule
CLIP embeddings        & 96.1 & 95.6 & 95.3 \\
ResNet-18 embeddings   & 95.9 & 95.4 & 95.1 \\
\bottomrule
\end{tabular}
\end{table}

\subsection{Representativeness as Distributional Faithfulness}
\label{subsec:mmd_appendix}

We provide an auxiliary evaluation to quantify whether our representativeness scoring induces subsets whose feature distribution better matches the full dataset.
Specifically, we measure \emph{distributional faithfulness} using Maximum Mean Discrepancy (MMD) computed in a fixed embedding space.

\paragraph{Embedding space.}
Unless otherwise stated, we compute features using CLIP ViT-L/14 image embeddings.
All embeddings are $\ell_2$-normalized before computing kernels.

\paragraph{Subset construction.}
For a given sampling ratio $p$, we compare two subset-selection strategies:
(i) \textbf{Ours (score-based)}: we rank examples by the precomputed representativeness score and select the top fraction \emph{within each class} (class-wise top-$p$), so that each class contributes approximately the same proportion of examples; and
(ii) \textbf{K-Center (geometric baseline)}: we run greedy $k$-center selection on the same embedding space and take the first $\lfloor p|\mathcal{D}|\rfloor$ selected samples.
The latter is purely geometry-based and does not enforce class balance.

\paragraph{MMD estimator.}
Let $\mathcal{V}=\{\phi(x_i)\}_{i=1}^{N}$ denote the full set of features and $\mathcal{U}=\{\phi(x)\mid x\in\mathcal{S}\}$ the selected subset features.
We use the squared MMD with an RBF kernel $k(\cdot,\cdot)$:
\begin{equation}
\mathrm{MMD}^{2}(\mathcal{U},\mathcal{V})
=
\mathbb{E}_{u,u'\sim \mathcal{U}}[k(u,u')]
+
\mathbb{E}_{v,v'\sim \mathcal{V}}[k(v,v')]
-
2\,\mathbb{E}_{u\sim \mathcal{U},\,v\sim \mathcal{V}}[k(u,v)].
\label{eq:mmd_mean}
\end{equation}
In practice, to control computational cost, we randomly subsample up to 5{,}000 examples from $\mathcal{U}$ and $\mathcal{V}$ to estimate these expectations.

\paragraph{Kernel setting.}
We follow the default setting in \texttt{sklearn.metrics.pairwise.rbf\_kernel}, i.e., $\gamma = 1/d$ where $d$ is the feature dimension, which corresponds to a fixed bandwidth in the chosen embedding space.

\paragraph{Protocol.}
We evaluate multiple sampling ratios $p\in\{0.01,0.05,0.1,0.3,0.5,0.7\}$.
For each $p$, we compute MMD between the selected subset and the full dataset.
Lower MMD indicates that the selected subset better preserves the full-data feature distribution in the chosen embedding space.
We visualize the trend in \cref{fig:mmd} and provide t-SNE overlays as qualitative complements.

\paragraph{Remark.}
This evaluation is not used for training and serves only as an analysis tool to assess distributional coverage induced by different selection strategies.

\subsection{Noise Robustness}
\label{sec:noiseRubust}
To evaluate robustness to imperfect supervision, we construct a noisy training set for CIFAR-100 by injecting \textbf{20\% symmetric label noise} into the training split while keeping the validation/test split unchanged (\cref{tab:noise_robustness_cifar100_r50_min}).
Concretely, we randomly sample 20\% training instances and replace their labels with a uniformly sampled class index in $\{0,\dots,99\}$, and then save the corrupted set in the original CIFAR binary format for a drop-in replacement during training.

We additionally visualize the per-sample usage statistics under label noise to reveal the dynamics of different selection strategies (\cref{fig:sample}).
For loss-driven methods, the usage distribution becomes highly concentrated, indicating that noisy samples with persistently high loss are repeatedly selected.
Such behavior amplifies sampling variance and introduces long-term bias in the effective training distribution.

Our method yields a markedly flatter usage curve, indicating more even sample participation over training (\cref{fig:sample}).
This stability arises from relying on semantics-aware scores computed offline and an explicit usage-frequency penalty, rather than noisy, step-wise loss feedback.
As a result, our approach mitigates selection-induced bias and maintains robust performance in noisy settings (\cref{tab:noise_robustness_cifar100_r50_min,fig:sample}).

\begin{figure}
    \centering
    \includegraphics[width=0.5\linewidth]{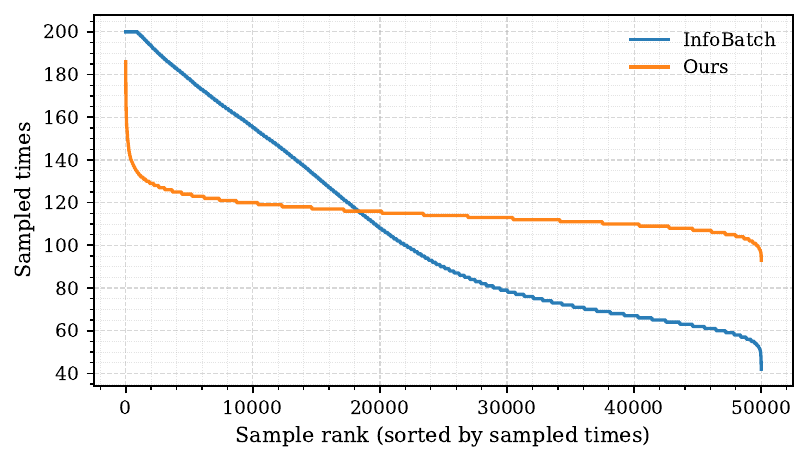}
    \caption{Per-sample usage statistics under 20\% symmetric label noise on CIFAR-100 (ResNet-50): samples are sorted by sampled times in descending order.
Loss-driven selection exhibits a more concentrated usage distribution, while our method yields a flatter curve, indicating more even sample participation and reduced long-term selection bias.}
    \label{fig:sample}
\end{figure}

% requires: \usepackage{booktabs}
% requires: \usepackage{booktabs}
\begin{table}[t]
\centering
\caption{Noise robustness on CIFAR-100 with ResNet-50 under 20\% label noise.
We report top-1 accuracy (\%) for clean and noisy training, and the drop $\Delta$ (Clean -- Noisy).}
\label{tab:noise_robustness_cifar100_r50_min}
\setlength{\tabcolsep}{8pt}
\begin{tabular}{@{}lccc|ccc@{}}
\toprule
\textbf{Setting}
& \multicolumn{3}{c}{\textbf{Selection ratio = 20\%}}
& \multicolumn{3}{c}{\textbf{Selection ratio = 30\%}} \\
\cmidrule(lr){2-4}\cmidrule(lr){5-7}
& \textbf{Clean} & \textbf{Noisy} & $\boldsymbol{\Delta}$
& \textbf{Clean} & \textbf{Noisy} & $\boldsymbol{\Delta}$ \\
\midrule
CLIP
& 54.10 & 46.05 & 8.05
& 61.90 & 58.34 & 3.56 \\
MoSo
& - & 31.01 & -
& 58.27 & 43.73 & 14.54 \\
MoDS
& 51.89 & 40.25 & 11.64
& 56.00 & 48.53 & 7.47 \\
InfoBatch
& 56.90 & 54.49 & 2.41
& 63.81 & 55.67 & 8.14 \\
\textbf{Ours}
& \textbf{64.29} & \textbf{60.33} & 3.96
& \textbf{75.70} & \textbf{60.86} & 14.84 \\
\bottomrule
\end{tabular}
\end{table}

%%%%%%%%%%%%%%%%%%%%%%%%TODO%%%%%%%%%%%%%%%%%%%%%%%%%%%%%%%%%%%%%%
\section{Experimental Details}
\subsection{Datasets and Tasks}

We evaluated the effectiveness of our approach on multiple image classification datasets, including CIFAR-10/100, Tiny-ImageNet, and ImageNet-1K.

\textbf{CIFAR-10/100.} Two classic image datasets, each comprising 50,000 training images and 10,000 test images. Images have a resolution of 32×32 and include 10/100 classes of natural images such as apples, bicycles, fish, etc.

\textbf{Tiny-ImageNet.} Tiny-ImageNet is a downscaled variant of the ImageNet dataset, containing 200 object categories. Each image is resized to 64 × 64 pixels and in RGB format. The dataset consists of 100,000 training images, 10,000 validation images, and 10,000 test images (with labels withheld in the official test set). It is commonly used for prototyping vision models under moderate computational budgets.
\begin{figure}[t]
\centering
\begin{subfigure}[b]{0.49\linewidth}
    \centering
    \includegraphics[width=\linewidth]{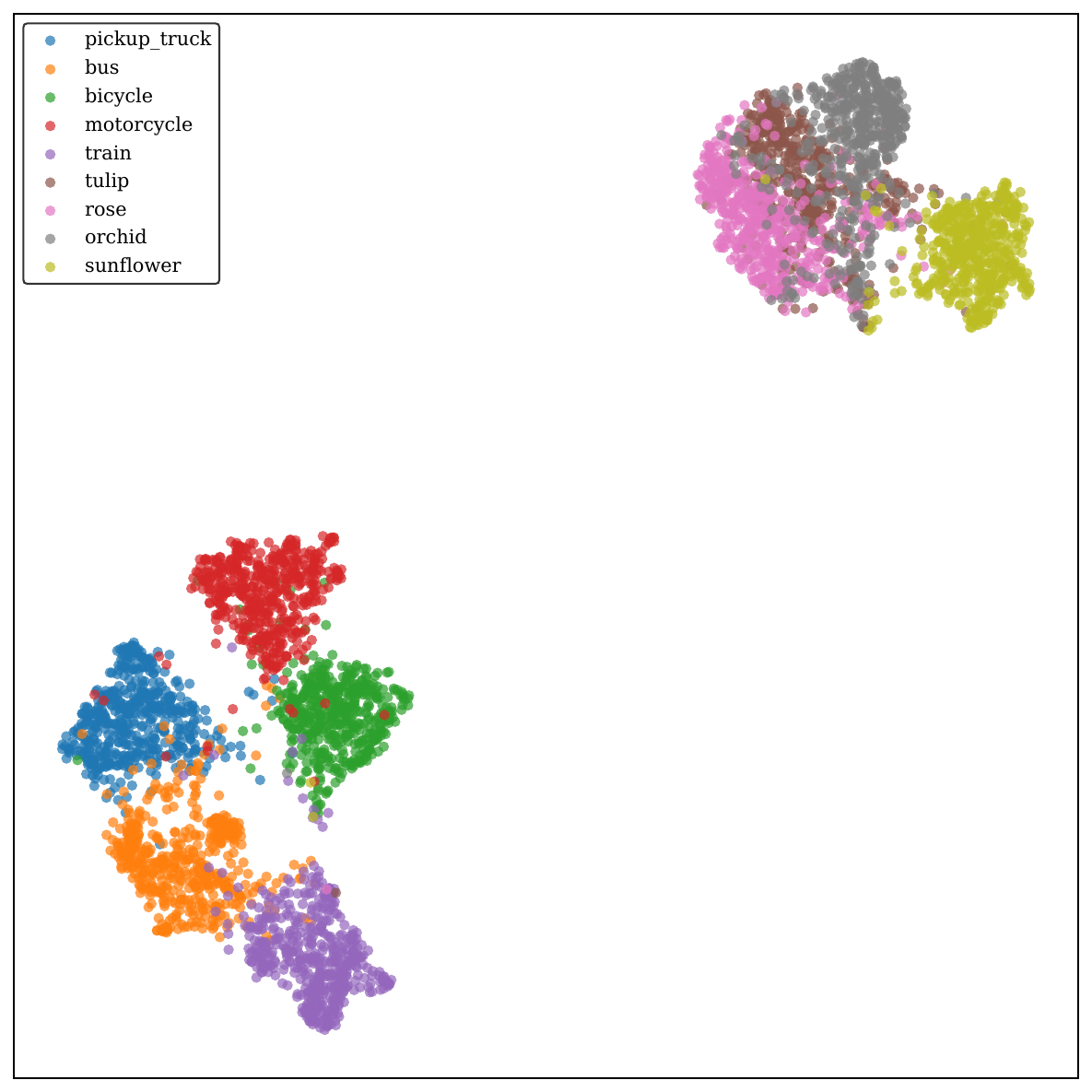}
    \caption{UMAP-latent space}
    \label{fig:umap1}
\end{subfigure}
\hfill
\begin{subfigure}[b]{0.49\linewidth}
    \centering
    \includegraphics[width=\linewidth]{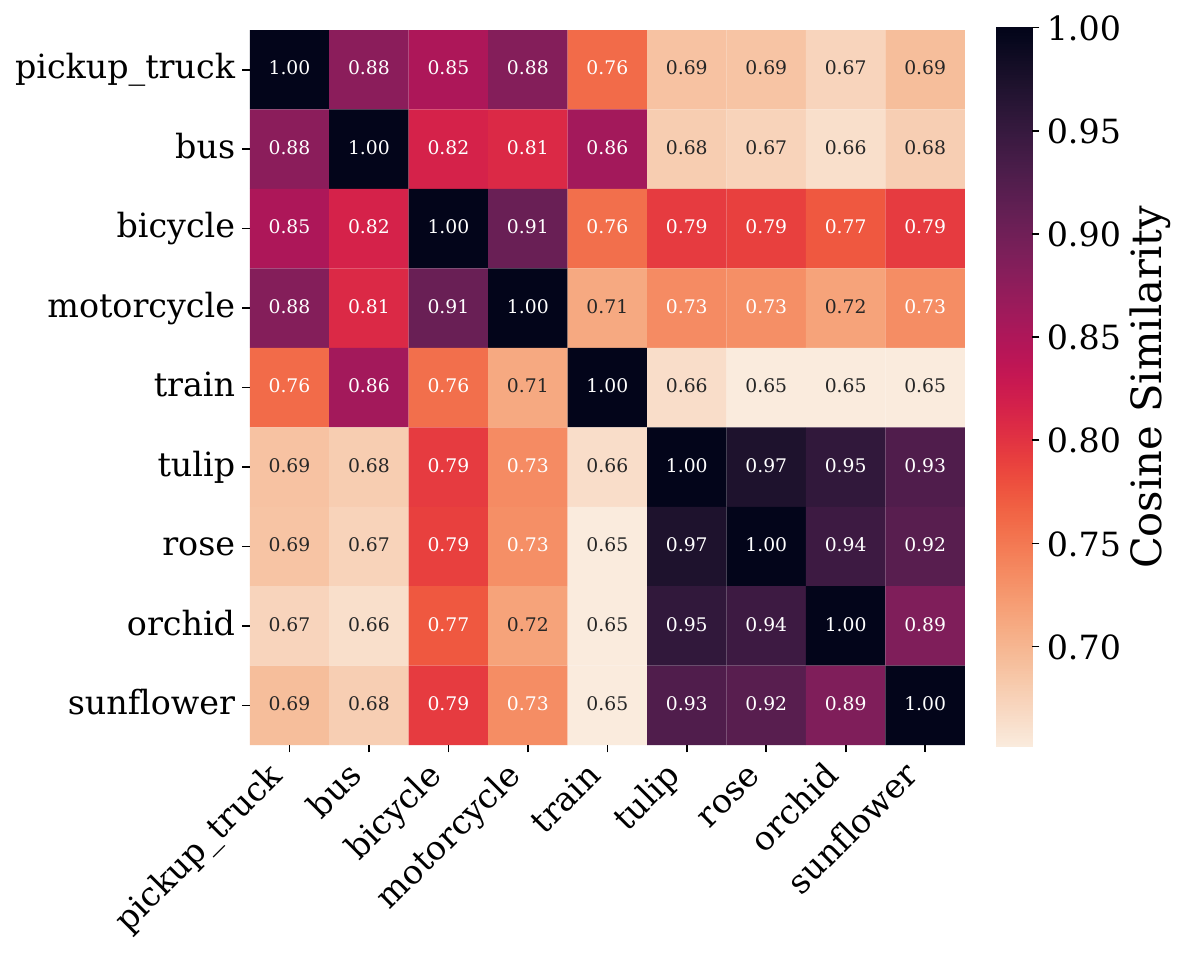}
    \caption{Class semantic relationship of SAE activations}
    \label{fig:umap2}
\end{subfigure}
\caption{Visualization of the SAE latent space. (a) presents the UMAP dimensionality reduction analysis of partial vehicle and plant classes. We observe that classes with similar semantic information are closely distributed in the SAE latent space, while those with high semantic discriminability are far apart. The heatmap in (b) further validates the successful training of our SAE.}
\end{figure}
\textbf{ImageNet-1K.} ImageNet-1K (also known as ILSVRC-2012) is a widely used benchmark dataset for large-scale image classification. It contains 1,000 object categories and over 1.2 million training images, along with 50,000 validation images and a held-out 100,000-image test set (labels unavailable publicly). Input images are of varying resolutions, typically center-cropped and resized to 224 × 224 pixels for standard evaluation. Due to its scale and diversity, ImageNet-1K serves as a cornerstone for evaluating representation learning, model architectures, and training strategies in computer vision.

\textbf{RSD\_15K.} RSD-15K is a large-scale, user-annotated dataset for suicide risk detection on social media. It contains roughly 15,000 user-level post histories, making it one of the largest datasets of its kind. Each user is associated with a complete chronological posting timeline to support analysis of risk evolution over time, and annotations follow a four-level risk grading scheme for finer-grained assessment. Labels were produced under expert supervision with multiple rounds of cross-validation to improve reliability and consistency.
\subsection{BackBones}
We evaluate our method on a diverse set of downstream architectures, spanning both convolutional networks and vision transformers, to verify that our selection scores transfer across model families.
For CNN backbones, we include ResNet-18/34/50 implemented with the standard BasicBlock/Bottleneck configurations, as well as VGG-16.
For transformer backbones, we consider two ViT variants: (i) a ViT-B/16 initialized from ImageNet-1K pretrained weights with the classification head replaced to match the target label space, and (ii) a lightweight ViT-Small tailored for $32\times 32$ images, trained under the same experimental protocol.
Unless otherwise noted, all backbones are trained with identical optimization settings, and only the per-epoch data stream is changed according to the selection strategy.

\begin{figure}
    \centering
    \includegraphics[width=1\linewidth]{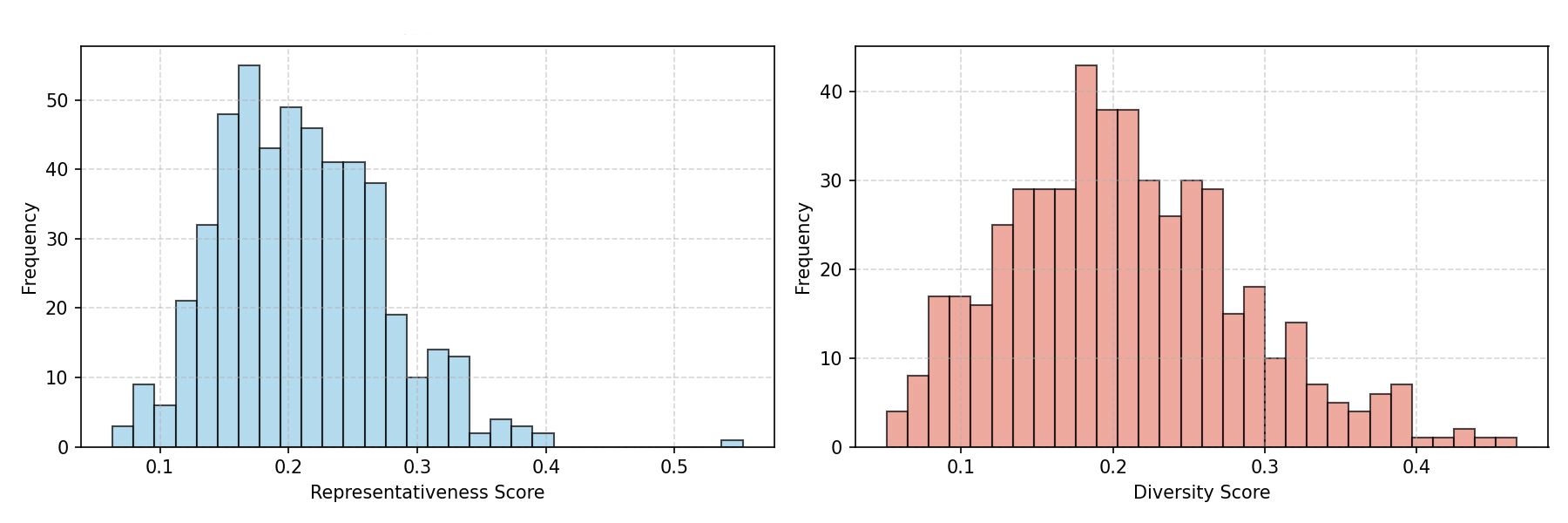}
    \caption{The score distribution of our method.}
    \label{fig:score}
\end{figure}

\subsection{Implement Details}
\label{sec:implement}
For experiments on CIFAR10/100 and TinyImageNet, we employed a single NVIDIA RTX A6000 GPU. All models were implemented from scratch following the official architectures, with dataset loading performed using the official torchvision library. Except for ViT, all implementations employed the SGD optimizer with momentum set to 0.9. For ResNet architectures, we applied weight decay of 5e-4, a batch size of 128, and an initial learning rate of 0.1. For VGG architectures, we used weight decay of 1e-4, a batch size of 256, and an initial learning rate of 0.1. None of the experiments utilized a warm-up. Unless otherwise specified, the Cos Annealing learning rate scheduler is employed.

For our proposed method, some parameters can be easily set. The minimum weights of the representativeness score $\alpha_{\min}=0.2$, $t_{mid}=0.6$, $k=0.05$, and $\lambda=0.2$ are default set if not specified. Unless otherwise specified, images have only undergone common augmentation operations such as normalization, random cropping, and horizontal flipping. Following previous research, we trained ResNet-18/50 on CIFAR-10/100, with an SGD optimizer with a momentum of 0.9, weight decay of $5e-4$, a batch size of 128, an initial learning rate of 0.1, and a total training epoch of 200 and VGG-16, with an SGD optimizer with a momentum of 0.9, a weight decay of $1e-4$, a batch size of 128, an initial learning rate of 0.1, and a total training epoch of 200. We trained ResNet-34 on ImageNet-1K using SGD with a momentum of 0.9, a weight decay of $5\times10^{-4}$, a batch size of 256, an initial learning rate of 0.1, and a total training epoch of 90.

%%%%%%%%%%%%%%%%%RSD Details%%%%%%%%%%%%%%%%%%%
We fine-tune roberta-base \cite{Liu2019RoBERTaAR} for RSD\_15K. 
We use the official train/test split, and further split the test set into validation and final test with a 50/50 stratified split for checkpoint selection and reporting.
Our method is instantiated as a static selector on the training split: we compute a combined score $s_i=\mathrm{Rep}(i)+\cdot \mathrm{Div}(i)$ and select the top 30\% training samples for fine-tuning (3 epochs, AdamW, lr $2\times10^{-5}$, batch size 16), while evaluation always uses the full validation/test data.

%%%%%%%%%%%%%%%%ImageNet-1K%%%%%%%%%%%%%%%%%%%
We evaluate on ImageNet-1K using \texttt{ImageFolder} with the standard train/val split.
Images are preprocessed with RandomResizedCrop$(224)$ and RandomHorizontalFlip for training, and Resize$(256)$ + CenterCrop$(224)$ for validation, followed by ImageNet normalization.
The downstream classifier is a ResNet-34 built with the standard \texttt{BasicBlock} design, trained for 90 epochs with SGD (momentum 0.9, weight decay $5\times 10^{-4}$, Nesterov) and CosineAnnealingLR ($T_{\max}=90$, $\eta_{\min}=10^{-6}$).
All runs use distributed training with \texttt{torchrun} and \texttt{FullyShardedDataParallel}; we adopt \texttt{DistributedSampler} for both train and val and aggregate validation metrics via all-reduce across ranks.
%%%%%%%%%%%%%%%%%Tiny ImageNet%%%%%%%%%%%%%%%%%%%
\subsection{Tiny-ImageNet Results}
\label{app:tiny}
For Tiny-ImageNet, we followed the settings of InfoBatch~\cite{qin2024infobatch}, using batch size 128, image
resize 64, SGD as the optimizer, and a learning rate of 0.1 at the start of cosine annealing. We compare against the SOTA methods on ResNet50 and VGG-16 respectively, and our approach surpasses existing SOTA methods in terms of both time efficiency and performance improvement.

% --- Tiny-ImageNet SOTA comparison table (ICML-ish style, no GPU time) ---
% Requires: \usepackage{booktabs}
\begin{table}[t]
\centering
\caption{\textbf{Tiny-ImageNet: comparison against backbone-specific SOTA baselines.}
For each backbone (ResNet-50 / VGG-16), we compare with the strongest prior method we reproduce.
\textbf{Cost} is the normalized training cost (Full-data $=1.00\times$), and \textbf{Speedup}$=1/\text{Cost}$.}
\label{tab:tiny_sota_backbone}
\setlength{\tabcolsep}{6pt}
\renewcommand{\arraystretch}{1.08}
\small
\begin{tabular}{l l c c}
\toprule
\textbf{Backbone} & \textbf{Method} & \textbf{Acc. (\%)} & \textbf{Speedup} \\
\midrule
\multirow{3}{*}{ResNet-50}
& Full-data & 63.5  & 1.00$\times$ \\
& InfoBatch & 63.4  & \textbf{1.71$\times$} \\
& \textbf{Ours} & \textbf{63.7} & 1.26$\times$ \\
\midrule
\multirow{3}{*}{VGG-16}
& Full-data & 58.21 & 1.00$\times$ \\
& CLIP & 56.34 & 1.10$\times$ \\
& \textbf{Ours} & \textbf{58.89}& \textbf{1.34$\times$} \\
\bottomrule
\end{tabular}
\vspace{2pt}
\end{table}

% --- helper macro for placeholders (optional) ---
% Put this in the preamble if you want a consistent placeholder style:
% \usepackage{xcolor}
% \newcommand{\placeholder}[1]{\textcolor{gray}{#1}}

\begin{figure*}
    \centering
    \includegraphics[width=1\linewidth]{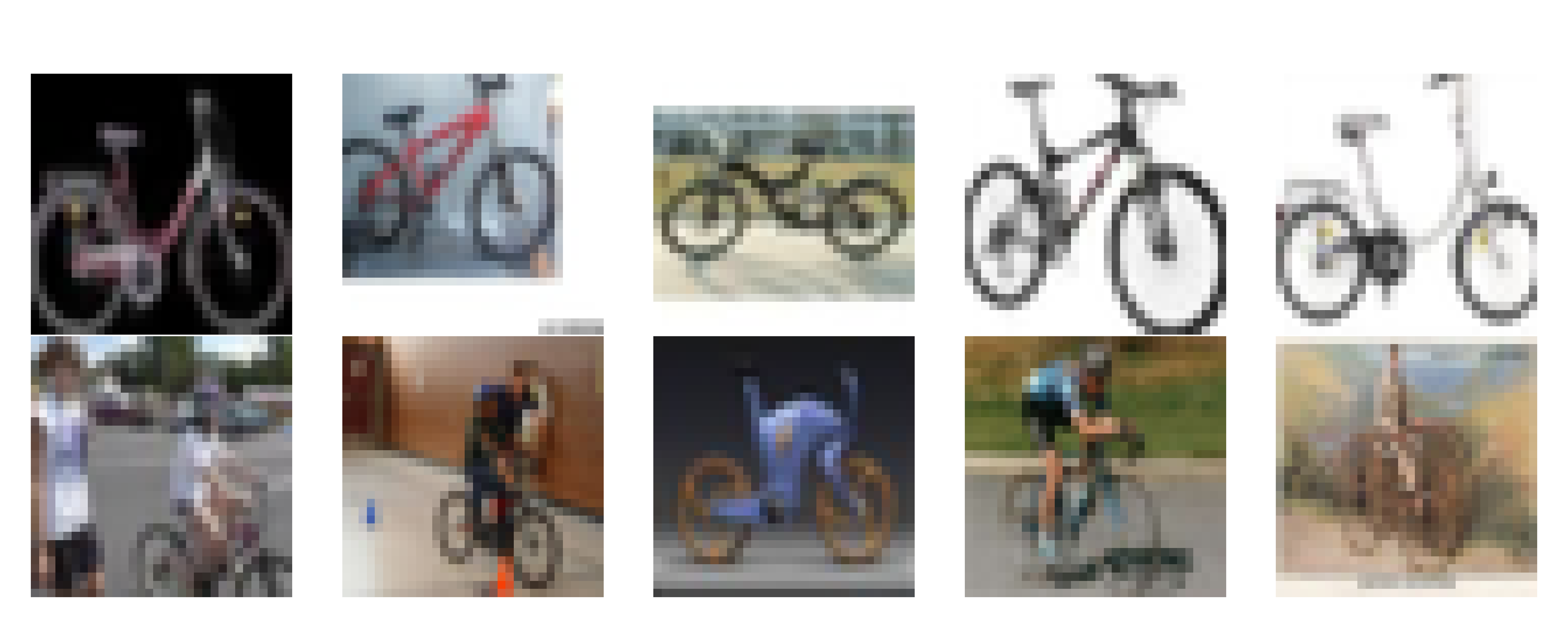}
    \caption{Visualization of the images of class 8 with the highest representativeness and diversity score. The top ones are of the highest representativeness score, while the bottom images display high-diversity samples.}
    \label{fig:bicycle}
\end{figure*}

\subsection{Robustness to Learning-Rate Schedulers}
\label{app:scheduler}

We report results under three widely used learning-rate schedulers: MultiStepLR, CosineAnnealingLR, and OneCycleLR.

\paragraph{MultiStepLR.}
We follow the standard configuration in~\cite{yang2024clip}, dividing the learning rate by $5$ at epochs $\{60,120,180\}$.

\paragraph{OneCycleLR.}
We follow the setup in~\cite{qin2024infobatch} and use the same total training budget.

\paragraph{Results.}
As shown in \cref{tab:scheduler}, our method maintains comparable (often slightly higher) accuracy to full-data training across all schedules, while consistently reducing training time.
These results indicate that the efficiency gains are stable under different learning-rate schedules and require no schedule-specific tuning.

\begin{table}[t]
\centering
\caption{\textbf{Top-1 accuracy (\%)} and training time (GPU hours) on CIFAR-100 (ResNet-18) under different learning-rate schedulers.
$\Delta$Acc denotes the accuracy difference relative to full-data training under the same scheduler. Lower time is better.}
\label{tab:scheduler}
\setlength{\tabcolsep}{6pt}
\begin{tabular}{@{}l c c c c@{}}
\toprule
\textbf{LR Scheduler} & \textbf{Method} & \textbf{Acc} & \textbf{$\Delta$Acc} & \textbf{Time} \\
\midrule
\multirow{2}{*}{MultiStep} & Full  & 77.3 & --   & 0.7 \\
                          & Ours  & \textbf{77.6} & +0.3 & \textbf{0.4} \\
\midrule
\multirow{2}{*}{OneCycle}  & Full  & 78.2 & --   & 0.7 \\
                          & Ours  & \textbf{78.2} & +0.0 & \textbf{0.4} \\
\midrule
\multirow{2}{*}{CosAnneal} & Full  & 78.2 & --   & 0.7 \\
                          & Ours  & \textbf{78.4} & +0.2 & \textbf{0.4} \\
\bottomrule
\end{tabular}
\end{table}

\begin{figure}[h]
    \centering
    \includegraphics[width=1\linewidth]{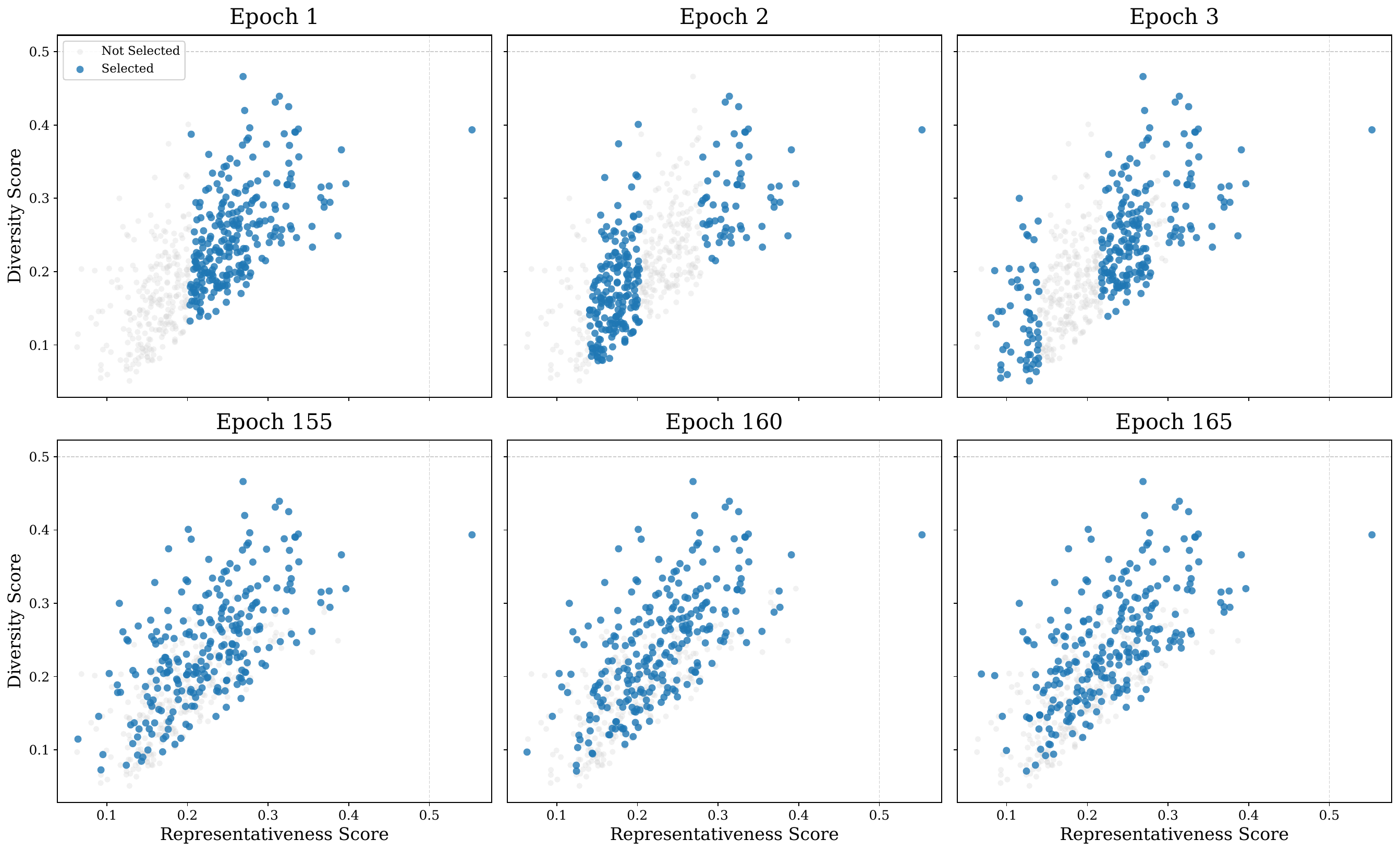}
    \caption{In the early stages of training, samples with high representativeness scores are primarily selected, while in the middle and late stages, samples are chosen by comprehensively considering both representativeness and diversity scores. The continuous epoch sampling visualization in the figure also intuitively demonstrates how our usage-frequency penalty functions.}
    \label{fig:ev}
\end{figure}

\section{Additional visualization}

\subsection{Selection Dynamics Visualization}

To gain deeper insight into our semantics-aware data selection, we conduct three complementary visual analyses.

First, temporal sampling dynamics \cref{fig:ev} track how samples from class 77 are selected across epochs. Early epochs (1–80) predominantly choose images with a high representativeness score, while later epochs (120–180) increasingly incorporate challenging cases manifesting high diversity. This evolution validates our design: the framework smoothly transitions from representativeness learning to boundary exploration. Simultaneously, our usage-based penalties play a crucial role across different stages. As shown in \cref{fig:ev}, our sampling pattern evolves with the sampling epoch—shifting from representation-based scoring to a balanced approach that integrates both representation and diversity scores. Concurrently, penalties introduce distinct samples within adjacent epochs to ensure sampling coverage, thereby safeguarding process-level diversity.

Secondly, our score distribution analysis reveals the unique characteristics of each metric. To eliminate dimensional differences, we normalized both scores using the min-max method before applying them to the data selection process. The figure below illustrates the distribution of our two scores. As seen, both scores exhibit near-normal distributions with means clustered around 0.2. This ensures a relatively unified scoring system, preventing any single score with an excessively large absolute value from disproportionately influencing the overall composite score.

\subsection{Qualitative Examples}
Analysis in \cref{fig:bicycle} of high-scoring samples revealed the semantic meanings underlying the ratings. The top five images showcase high-representativeness score samples, while the bottom images display high-diversity score samples. The most representative bicycle images consistently feature a frontal perspective, clear outlines, simple backgrounds, and unobstructed views of the entire vehicle—core characteristics of this category. Conversely, the most diverse samples exhibit semantic variability: obscured bicycles, cyclists on bikes, blurred images, and complex backgrounds. These are not outliers but marginally valuable cases that optimize decision boundaries.

\section{Related Work}
\label{sec:related2}
\textbf{Data Selection.} Existing data selection methods can be categorized into static and dynamic approaches based on their selection frequency. 

Static methods typically evaluate sample importance offline and compute a corresponding score for each sample in a one-off manner using predefined criteria, such as representativeness, diversity, importance, and difficulty. Relevant works include K-Center \cite{sener2018activelearningconvolutionalneural}, one of the primary baselines we compare against in the main text, which performs sample selection by minimizing the maximum distance from each sample to its nearest cluster centroid. Herding \cite{10.1145/1553374.1553517} is another geometry-based static selection method that selects samples whose distribution is close to that of the full dataset by minimizing the distance between the subset centroid and the full dataset centroid. Early static methods usually evaluate samples along a single dimension, while methods represented by G-DIG \cite{pan2024gdiggradientbaseddiversehighquality} jointly consider sample importance and diversity during subset selection—their diversity control is ensured via gradient clustering and random sampling after screening samples by representativeness. The CLIP \cite{yang2024clip} method treats diversity as an indicator equally important to representativeness, an approach that effectively prevents the a priori bias of subset diversity caused by considerations of importance or representativeness alone. InfoMax \cite{tan2025data} formulates data selection as a discrete quadratic programming problem of maximizing information and minimizing redundancy, where sample importance quantifies the information content of individual samples and pairwise sample similarity measures information overlap, thus achieving a balance between importance and diversity. This further explores the way to balance the roles of representativeness and diversity in data selection.

Dynamic methods generally achieve a better trade-off between training acceleration and accuracy. Current dynamic methods primarily assign dynamic scores to samples using instantaneous model signals. UCB and $\epsilon$-Greedy \cite{raju2021acceleratingdeeplearningdynamic} select samples that are challenging for the current model based on sample uncertainty, while InfoBatch \cite{qin2024infobatch} focuses on sample loss information and dynamically selects samples with high loss values for training. Leveraging instantaneous model signals helps identify samples that are difficult for the model in its current state, yet it undoubtedly also increases the likelihood of sampling noisy samples on noisy datasets. Distinct from existing methods, we are the first to introduce representativeness and diversity scoring into dynamic data selection, and we design a downstream model-agnostic dynamic data selection framework via a combination of offline scoring and online scheduling. Meanwhile, we measure representativeness at the feature level and diversity at the process level, and we innovatively introduce a usage-frequency penalty to facilitate sample rotation. The advantages and detailed implementation of our method are elaborated in the main text and are not repeated here.

\textbf{Sparse Autoencoders.} Sparse autoencoders are unsupervised representation learners\cite{gao2024scalingevaluatingsparseautoencoders} that impose explicit sparsity constraints—e.g., L1 regularization or hard k-sparse activation—on hidden units, forcing the network to reconstruct inputs using only a small subset of neurons. This induces localized, disentangled, and highly discriminative feature bases, where individual neurons tend to specialize in coherent semantic concepts (e.g., “wheel”, “fur texture”) rather than entangled low-level patterns. As a result, SAEs substantially improve interpretability and generalization while intrinsically mitigating overfitting—making them especially well-suited for compressing and structuring high-dimensional, redundant embeddings. Since the advent of K-Sparse Autoencoders\cite{makhzani2014ksparseautoencoders} (K-SAE) in 2013, SAEs have been widely adopted in image denoising, anomaly detection, and representation pretraining. More recently, they have emerged as a cornerstone in large model interpretability: researchers decompose LLM internal activations into human-annotatable “semantic neurons”\cite{cunningham2023sparseautoencodershighlyinterpretable} (e.g., “geographic entity”, “syntactic role”) and enable targeted concept editing. The data selection method based on SAE is not the first of its kind. Previous research\cite{yang2025diversitydrivendataselectionlanguage} demonstrated that SAE can be applied to data selection by maximizing the number of activated neurons in a subset, thereby ensuring the diversity of the selected subset. Lou's work\cite{lou2025saevinterpretingmultimodalmodels} was the first to extend the SAE-based data selection method to the multimodal domain, where SAE was used to compare the degree of image-text alignment among samples.

% \begin{wrapfigure}{r}{0.42\linewidth}
%   \vspace{-6pt}
%   \centering
%   \includegraphics[width=\linewidth]{fig/lambda.png}
%   \caption{The visualization of the impact of $\lambda$ on data sampling}
%   \label{fig:lambda}
%   \vspace{-10pt}
% \end{wrapfigure}

\end{document}